\documentclass{article} 
\pdfoutput = 1
\usepackage{algorithm} 
\usepackage{algpseudocode}
\usepackage{amssymb}
\usepackage{amsthm}
\usepackage{graphicx}
\usepackage{amsfonts}
\usepackage{appendix}
\usepackage{mathrsfs}
\usepackage{amsmath}
\usepackage[numbers]{natbib}
\usepackage{comment}
\usepackage{subfigure}
\usepackage{enumerate}

\newtheorem{theorem}{\textbf{Theorem}}
\newtheorem{cor}[theorem]{\textbf{Corollary}}
\newtheorem{lem}[theorem]{\textbf{Lemma}}

\newtheorem{definition}{\textbf{Definition}}

\newcommand{\ml}{\mathscr{L}}
\newcommand{\tml}{\tilde{\mathscr{L}}}
\newcommand{\mpp}{P}
\newcommand{\mb}{B}

\newcommand{\mr}{R}

\newcommand{\md}{\mathscr{D}}
\newcommand{\tmd}{\tilde{\mathscr{D}}}
\newcommand{\tp}{\tilde P}

\newcommand{\modinv}{\mone}
\newcommand{\vodinv}{\mtwo}

\newcommand{\hatd}{\hat d}
\newcommand{\hatdt}[1]{\hat{d}_{#1,\tau}}
\newcommand{\sd}{d}
\newcommand{\tsd}{\tilde{d}}

\newcommand{\sdt}[1]{d_{#1,\tau}}
\newcommand{\mind}{d_{min,n}}

\newcommand{\taun}{\tau_{min}}

\newcommand{\cldeg}[1]{\gamma_{#1,n}}
\newcommand{\cldegt}[1]{\gamma_{#1}}

\newcommand{\cent}{\textsf{cent}}

\newcommand{\popv}{\mathscr{V}}

\newcommand{\vsym}{\mathscr{V}}
\newcommand{\tvsym}{\tilde{\mathscr{V}}}

\newcommand{\usym}{V}

\newcommand{\sdgen}{d}

\newcommand{\one}{\mathbf{1}}

\newcommand{\egs}{\lambda}
\newcommand{\beg}{B_{eig}}
\newcommand{\tmbb}{\tilde{B}}
\newcommand{\tmb}{\tilde{B}_{eig}}

\newcommand{\kw}{K_w}
\newcounter{savecntr}
\newcounter{restorecntr}
\newcommand{\maxd}{d_{max,n}}

\newcommand{\nn}[1]{n_{#1}}
\newcommand{\mone}{m_{1,n}}
\newcommand{\tmone}{\tilde{m}_{1,n}}
\newcommand{\mtwo}{m_{2,n}}

\newcommand{\tnn}[1]{ \tilde{n}_{#1} }

\newcommand{\objec}{\textsf{Obj}}
\newcommand{\hatt}{\hat{\mathcal{T}}}
\newcommand{\ttt}{\mathcal{T}}

\newcommand{\tepst}{\tilde{\epsilon}_{\tau,n}}
\newcommand{\epst}{{\epsilon}_{\tau,n}}

\newcommand{\ddelta}{\delta}
\newcommand{\ddeltn}{\delta_{\tau,n}}
\newcommand{\ddeln}{\delta_{n}}
\newcommand{\tddeln}{\tilde{\delta}_{n}}
\newcommand{\pkn}{p_n^s}
\newcommand{\nk}{n^s}

\newcommand{\XX}{X}
\newcommand{\MM}{M}
\newcommand{\dsn}{d^s_n}
\newcommand{\dwn}{d^w_n}
\title{Impact of regularization on Spectral Clustering}

\author{Antony Joseph\setcounter{savecntr}{\value{footnote}}\thanks{Department of Genome Dynamics, Lawrence Berkeley National Laboratory, and Department of Statistics, University of California, Berkeley. email: AntonyJoseph@lbl.gov} and Bin Yu\setcounter{restorecntr}{\value{footnote}}\thanks{Department of Statistics and EECS,  University of California, Berkeley. email: binyu@stat.berkeley.edu}
}

\begin{document}

\maketitle

\begin{abstract}
The performance of spectral clustering can be considerably improved via regularization, as demonstrated empirically in \citet{chen2012fitting}. 
 Here, we provide an attempt at quantifying this improvement through theoretical analysis.  Under the stochastic block model (SBM), and its extensions, previous results on spectral clustering  relied on the minimum degree of the graph being sufficiently large for its good performance. By examining the scenario where the regularization parameter $\tau$ is large we show that the minimum degree assumption can potentially be removed. As a special case, for an SBM with two blocks, the results require the maximum degree to be large (grow faster than $\log n$) as opposed to the minimum degree. 
  More importantly,  we show the usefulness of regularization in situations where not all nodes belong to well-defined clusters. Our results rely on a `bias-variance'-like trade-off that arises from understanding the concentration of the sample Laplacian and the eigen gap as a function of the regularization parameter. 
 As a byproduct of our bounds, we propose a data-driven technique \textit{DKest} (standing for estimated
Davis-Kahan bounds) for choosing the regularization parameter. This technique is shown to work well through simulations and on a real data set.
\end{abstract}

\section{Introduction}

The problem of identifying communities (or clusters) in large networks is an important contemporary problem in statistics.  
  Spectral clustering is one of the more popular techniques for such a purpose, chiefly due to its computational advantage and generality of application. The algorithm's  generality arises from the fact that it is not tied to any modeling assumptions on the data, but is rooted in intuitive measures of community structure such  as \textit{sparsest cut} based measures  \cite{hagen1992new}, \cite{shi2000normalized}, \cite{kwok2013improved}, \cite{ng2002spectral}.  Other examples of applications of spectral clustering include manifold learning  \cite{belkin2003laplacian}, image segmentation \cite{shi2000normalized}, and text mining \cite{dhillon2001co}.
  
 
The canonical nature of spectral clustering also generates interest in variants of the technique. Here, we attempt to better understand the impact of regularized forms of spectral clustering for community detection in networks. In particular, we focus on the regularized spectral clustering (RSC) procedure proposed in \citet{chen2012fitting}. Their empirical findings  demonstrates that the performance of the RSC algorithm, in terms of obtaining the correct clusters, is significantly better for certain values of the regularization parameter. An alternative form of regularization was studied in  \citet{chaudhuri2012spectral} and \citet{qin2013regularized}. 
 

This paper provides an attempt to provide a theoretical understanding for the regularization in the RSC algorithm. We also propose a practical scheme for choosing the regularization parameter based on our theoretical results. Our analysis  focuses on the Stochastic Block Model (SBM) and an extension of this model. 	Below are the three main contributions of the paper.
 
 

\begin{enumerate}[(a)]
\item We attempt to understand regularization for the stochastic block model. In particular, for a graph with $n$ nodes, previous theoretical analyses for spectral clustering, under the SBM and its extensions, \cite{rohe2011spectral},\cite{chaudhuri2012spectral}, \cite{sussman2012consistent}, \cite{fishkind2013consistent} assumed that the minimum degree of the graph scales at least by a polynomial power of $\log n$. Even when this assumption is satisfied, the dependence on the minimum degree is highly restrictive when it comes to making inferences about cluster recovery. Our analysis provides cluster recovery results that potentially do not depend on the above mentioned constraint on the minimum degree. As an example, for an SBM with two blocks (clusters), our results require that the 
 maximum degree be large (grow faster than $\log n$) rather than the minimum degree. This is done in Section \ref{sec:examples}.
\item  We demonstrate that regularization has the potential of addressing a situation where the lower degree nodes do not belong to well-defined clusters.
Our results demonstrate that choosing a large regularization parameter has the effect of removing these relatively lower degree nodes. Without regularization, these nodes would hamper with the clustering of the remaining nodes in the following way:  In order for spectral clustering to work, the top eigenvectors - that is, the eigenvectors corresponding to the largest eigenvalues of the Laplacian -   need to be able to discriminate between the clusters. Due to the effect of   nodes that do not belong to well-defined clusters these top  eigenvectors do not necessarily discriminate between the  clusters with ordinary spectral clustering. This is done in Section \ref{sec:selhigh} 
\item  Although our theoretical results deal with the `large' $\tau$ case, it is observed empirically that moderate values of $\tau$ may produce  better clustering performance. Consequently,  in Section \ref{sec:simresults} we propose $DKest$, a data dependent procedure for choosing the regularization parameter.   We demonstrate that this works well through simulations and  on a real data set. This is in Section \ref{sec:simresults}.
\end{enumerate}

Our theoretical results involve understanding the trade-offs between the \textit{eigen gap} and the concentration of the sample Laplacian when viewed as a function of the regularization parameter.  Assuming that there are $K$ clusters, the eigen gap refers to the gap between the  $K$-th smallest eigenvalue and the remaining eigenvalues.  An adequate  gap  ensures that the sample eigenvectors can be estimated well (\cite{von2007tutorial}, \cite{ng2002spectral}, \cite{kwok2013improved}) which leads to good cluster recovery. The adequacy of an eigen gap for cluster recovery is in turn determined by the concentration of the sample Laplacian. 

In particular, a consequence of the Davis-Kahan theorem \cite{bhatia1997matrix} is that if the spectral norm of the difference of the sample and population Laplacians is small compared to the eigen gap then the top $K$ eigenvector can be estimated well.   
  Denoting $\tau$ as the regularization parameter, previous theoretical analyses of regularization (\cite{chaudhuri2012spectral}, \cite{rohe2011spectral}) provided high-probability bounds on this spectral norm. These bounds have a $1/\sqrt{\tau}$ dependence on $\tau$, for large $\tau$. In contrast, our high probability bounds behave like  $1/\tau$, for large $\tau$. We also demonstrate that the eigen gap behaves like $1/\tau$ for large $\tau$. The end result is that we show that one can get a good understanding of the impact of regularization by understanding the situation where $\tau$ goes to infinity. This also explains  empirical observations in \cite{chen2012fitting}, \cite{qin2013regularized} where it was seen that performance of regularized spectral clustering does not change for  $\tau$ beyond a certain value.  Our  procedure for choosing the regularization parameter works by providing estimates of the Davis-Kahan bounds over a grid of values of $\tau$ and then choosing the $\tau$ that minimizes these estimates.




The paper is divided as follows. In the next subsection we discuss preliminaries. In particular,  in Subsection \ref{subsec:regulspec}  we review the RSC algorithm of \cite{chen2012fitting}, and also discuss the other forms of regularization in literature. In Section \ref{sec:stochasticblk} we review the stochastic block model.  Our theoretical results, described in (a) and (b) above, are provided in Sections \ref{sec:examples} and \ref{sec:selhigh}.  Section \ref{sec:simresults} describes our $DKest$ data dependent method for choosing the regularization parameter.


%
%
%
%
%
%

\subsection{Regularized spectral clustering} \label{subsec:regulspec}

In this section we review the regularized spectral clustering (RSC) algorithm  of \citet{chen2012fitting}.

We first introduce some basic notation. A graph with $n$ nodes and edge set $E$ is represented by the $n \times n$ symmetric adjacency matrix $A = ((A_{ij}))$, where $A_{ij} = 1$ if there is an edge between $i$ and $j$, otherwise $A_{ij}$ is 0. In other words, for $1 \leq i,\, j \leq n$, 
$$A_{ij} = \begin{cases} 1, & \mbox{if } (i,\,j) \in E \\ 0, & \mbox{otherwise }  \end{cases}.$$

Given such a graph, the typical community detection problem is synonymous with finding a partition of the nodes. A good partitioning would be one in which there are  fewer edges between the various components of the partition, compared to the number of edges within the components. Various measures for goodness of a partition have been proposed, chiefly the Ratio Cut \cite{hagen1992new} and Normalized  Cut \cite{shi2000normalized} . However, minimization of the above measures is an NP-hard  problem since it involves searching over all partitions of the nodes. The significance of spectral clustering partly arises from the fact that it provides a continuous approximation to the above discrete optimization problem \cite{hagen1992new}, \cite{shi2000normalized}. 


We now describe the RSC algorithm  \cite{chen2012fitting}. Denote by $D= diag(\hatd_1, \ldots, \hatd_n)$ the diagonal matrix of degrees, where $\hatd_i = \sum_{j = 1}^n A_{ij}$. The normalized (unregularized) symmetric graph Laplacian is defined as
$$L = D^{-1/2} A D^{-1/2}.$$

Regularization is introduced in the following way:  Let $J$ be a constant matrix with all entries equal to $1/n$. Then, in regularized spectral clustering one constructs a new adjacency matrix by adding $\tau J$ to the adjacency matrix $A$ and computing the corresponding Laplacian.  In particular, let
$$A_{\tau} = A + \tau J,$$
where $\tau > 0$ is the regularization parameter. The corresponding regularized symmetric Laplacian is defined as
\begin{equation}
L_\tau = D_{\tau}^{-1/2} A_{\tau}D_{\tau}^{-1/2}. \label{eq:regrsc}
\end{equation}

 Here, $D_{\tau} = diag(\hatdt{1},\, \ldots, \hatdt{n})$ is the diagonal matrix of `degrees' of the modified adjacency matrix $A_{\tau}$. In other words, $\hatdt{i} = \hatd_i + \tau$.

 The RSC algorithm for finding $K$ communities is described in  Algorithm \ref{fig:psc}. In order to bring to the forefront the dependence on $\tau$, we also denote the RSC algorithm as RSC-$\tau$.  The algorithm first computes $\usym_\tau$,  the $n\times K$ eigenvector matrix corresponding to the $K$ largest  eigenvalues of $L_\tau$. The columns of $\usym_\tau$ are taken to be orthogonal. 
 The rows of $\usym_\tau$, denoted by $\usym_{i,\tau}$, for $i =1,\ldots,n$, corresponds to the nodes in the graph. Clustering the rows of $\usym_\tau$, for example using the $K$-means algorithm,  provides a clustering of the nodes. 
  We remark that the RSC-0 Algorithm corresponds to the usual spectral clustering algorithm. 
  
 
 \begin{algorithm}
 \begin{algorithmic}
 \item \textbf{Input :} Laplacian matrix $L_\tau$.
\State \textbf{Step 1:} Compute the $n\times K$ eigenvector matrix $\usym_\tau$.
\State \textbf{Step 2:}  Use  the $K$-means algorithm to cluster the rows of $\usym_\tau$  into  $K$ clusters. 
 \end{algorithmic}
 \caption{The RSC-$\tau$ Algorithm \cite{chen2012fitting}}
 \label{fig:psc}
\end{algorithm}

Our theoretical results assume that the data is randomly generated from a stochastic block model (SBM), which we review in the next subsection. While it is well known that there are real data examples where the SBM fails to provide a good approximation, we believe that the above provides a good playground for understanding the role of regularization in the RSC algorithm. Recent works \cite{chen2012fitting}, \cite{fishkind2013consistent}, \cite{rohe2011spectral},   \cite{bickel2009nonparametric},  \cite{karrer2011stochastic} have used this model, and its variants, to provide a theoretical analyses for various community detection algorithms.

In   \citet{chaudhuri2012spectral},  the following alternative regularized version of the symmetric Laplacian is proposed:
\begin{equation}
L_{deg, \tau} = D_\tau^{-1/2} A D_\tau ^{-1/2}. \label{eq:chaudhari}
\end{equation}
 Here, the subscript $deg$ stands for `degree' since the usual Laplacian is modified by adding $\tau$ to the degree matrix $D$.
Notice that for the RSC algorithm the matrix $A$ in the above expression was replaced by $A_\tau$. 

As mentioned before, we attempt to understand regularization in the framework of the SBM 
and its extension. We review the SBM in the next section. Using recent results on the concentration of random graph Laplacians \cite{oliveira2009concentration},  we were able to show concentration results in Theorem \ref{thm:someassump} for the regularized Laplacian in the RSC algorithm.  Previous concentration results for the Laplacian \eqref{eq:chaudhari}, as in \cite{chaudhuri2012spectral}, provide high probability bounds on the spectral norm of the difference of the sample and population regularized Laplacians that  depends inversely on $1/\sqrt{\tau}$. 
However, for the regularization \eqref{eq:regrsc} we show that the dependence is inverse in $\tau$, for large $\tau$. We believe that this holds for the regularization \eqref{eq:chaudhari} as well. We also demonstrate that the eigen gap depends inversely on $\tau$, for large $\tau$. The benefit of this, along with our improved concentration bounds, is that one can understand regularization by looking at the case where $\tau$ is large. This results in a very neat criterion for the cluster recovery with the RSC-$\tau$ algorithm.


\section{The Stochastic Block Model} \label{sec:stochasticblk}

Given a set of $n$ nodes, the stochastic block model (SBM), introduced in \cite{holland1983stochastic}, is one among many   random graph models that has  communities inherent in its definition. 
We denote the number of communities in the SBM by $K$. Throughout this paper we assume that $K$ is known. The communities, which represent a partition of the $n$ nodes, are assumed to be fixed beforehand. Denote these by $C_1,\, \ldots,\, C_K$.  Let $\nn{k}$, for $k=1,\ldots,K$, denote the number of nodes belonging to each of the clusters. 

Given the communities, the edges between  nodes,  say $i$ and $j$, are chosen independently with probability depending the communities $i$ and $j$ belong to.  In particular, for a node $i$ belonging to cluster $C_{k_1}$, and node $j$ belonging to cluster $C_{k_2}$, the probability of edge between $i$ and $j$  is given by
\begin{equation*}
\mpp_{ij} = \mb_{k_1, k_2}. 
\end{equation*}
Here, the \textit{block probability matrix} 
$$\mb = ((\mb_{k_1, k_2})), \quad \text{where $k_1,\, k_2 = 1,\ldots, K$}$$ is a symmetric full rank matrix, with each entry between  $[0,1]$. The $n\times n$ edge probability  matrix $\mpp = ((P_{ij}))$, given by \eqref{eq:estimatedprob}, represents the population counterpart of the adjacency matrix $A$. 

Denote $Z = ((Z_{ik}))$ as the $n\times K$ binary matrix providing the cluster  memberships of each node. In other words, each row of $Z$ has exactly one 1, with $Z_{ik} = 1$ if node $i$ belongs to $C_k$. Notice that, 
\begin{equation}
\mpp = Z\mb Z'.\label{eq:estimatedprob}
\end{equation}
Here $Z'$ denotes the transpose of $Z$. Consequently, from \eqref{eq:estimatedprob}, it is seen that the rank of $\mpp$ is also $K$.



The population counterpart for the degree matrix $D$ is denoted by $\md = diag(\sd_1 ,\ldots, \sd_n)$, where
 $\md = diag(\mpp \mathbf{1})$. Here $\mathbf{1}$ denotes the column vector of all ones.
    Similarly, the population version of the symmetric Laplacian $L_\tau$ is denoted by $\ml_\tau$, where 
$$  \ml_\tau = \md_\tau^{-1/2} \mpp_\tau \md_\tau^{-1/2}.$$
Here $\md_\tau = \md + \tau I$ and $\mpp_\tau = \mpp + \tau J.$ The $n \times n$ matrices $\md_\tau$ and $\mpp_\tau$ represent the population counterparts to $D_\tau$ and $A_\tau$ respectively. Notice that since $\mpp$ has rank $K$, the same holds for $\ml_\tau$. 

 \subsection{Notation}

We use $\|.\|$ to denote the spectral norm of a matrix. Notice that for vectors this corresponds to the usual $\ell_2$-norm. We use $A'$ to denote the transpose of a matrix, or vector, $A$.

For positive $a_n, \, b_n$, we use the notation $a_n \asymp b_n$ if there exists universal constants $c_1,\, c_2> 0$ so that $c_1 a_n \leq b_n \leq c_2 a_n$. Further, we use $b_n \lesssim a_n$ if $b_n \leq c_2 a_n$, for some positive $c_2$ not depending on $n$.  The notation $b_n \gtrsim a_n$ is analogously defined.
 
The quantities
  $$\mind = \min_{i= 1,
 \ldots,n}\sd_i, \quad\quad\quad \maxd = \max_{i= 1,
 \ldots,n}\sd_i$$
  denote the minimum and maximum expected degrees of the nodes.

\subsection{The Population Cluster Centers} \label{subsec:popclus}

We now proceed to define population cluster centers $\cent_{k,\tau} \in \mathbb{R}^K$, for $k = 1,\ldots, K$, for the $K$ block SBM. These points are defined so that  the rows of the eigenvector matrix $V_{i,\tau}$, for $i\in C_k$, are expected to be scattered around  $\cent_{k,\tau}$.  

Denote  by $\vsym_\tau$ an $n\times K$ matrix containing the  eigenvectors of  the $K$ largest eigenvalues of the population  Laplacian $\ml_\tau$.  
As with $\usym_\tau$, the columns of $\vsym_\tau$ are also assumed to be orthogonal. 

Notice that both $\vsym_\tau$ and $-\vsym_\tau$ are eigenvector matrices corresponding to $\ml_\tau$.  This ambiguity in the definition of $\vsym_\tau$ is further complicated if an eigenvalue of $\ml_\tau$ has multiplicity greater than one.  We do away with this ambiguity in the following way: Let $\mathcal{H}$ denote the set of all $n\times K$ eigenvector matrices of $\ml_\tau$ corresponding to the top $K$ eigenvalues.  We take,
\begin{equation}
  \popv_\tau = \arg\min_{H \in \mathcal{H}}\|V_\tau - H \|, \quad  \label{eq:contv}
\end{equation}
where recall that $\|.\|$ denotes the spectral norm. 
The matrix $\popv_\tau$, as defined above, represents the population counterpart of the matrix $\usym_\tau$. 

 Let $\vsym_{i,\tau}$ denote the $i$-th row of $\vsym_\tau$. Notice that since the set $\mathcal{H}$ is closed under the $\|.\|$ norm, one has that $\popv_\tau$ is also an eigenvector matrix of $\ml_\tau$ corresponding to the top $K$ eigenvalues. Consequently, the rows $\vsym_{i,\tau}$ are the same  across nodes belonging to a particular cluster (See, for example, \citet{rohe2011spectral} for a proof of this fact). In other words, there are $K$ distinct rows of $\vsym_{i,\tau}$, with each row corresponding to nodes from one of the $K$ clusters. 
 
 Notice that the matrix $\vsym_{i,\tau}$ depends on the sample eigenvector matrix $V_\tau$ through \eqref{eq:contv}, and consequently is a random quantity.  However, the following lemma shows that the pairwise distances between the rows of $\vsym_{i,\tau}$ are non-random and, more importantly,  independent of $\tau$.

\begin{lem} \label{lem:rescue} Let $i \in C_k$ and $i' \in C_{k'}$. Then,
\begin{equation*}
\|\vsym_{i,\tau}  - \vsym_{i',\tau} \| = \begin{cases} 0, & \mbox{if }  k = k' \\ \sqrt{\frac{1}{\nn{k}} + \frac{1}{\nn{k'}}}, & \mbox{if } k \neq k' \end{cases}
\end{equation*}
\end{lem}

From the above lemma, there are  $K$ distinct rows of $\vsym_\tau$ corresponding to the $K$ clusters.   We denote these  as  $\cent_{1,\tau},\ldots, \cent_{K,\tau}$. We also call these the population cluster centers since, intuitively, in an idealized scenario the data points $V_{i,\tau}$, with $i \in C_k$, should be concentrated around $\cent_{k,\tau}$.

 


%

\subsection{Cluster recovery using $K$-means algorithm} \label{subsec:relate}


Recall that the RSC-$\tau$ Algorithm \ref{fig:psc} works by performing $K$-means clustering on the rows of the $n\times K$ sample eigenvector matrix, denoted by $V_{i,\tau}$, for $i  =1,\ldots,n$.  In this section, in particular Corollary \ref{cor:simcor}, we relate the fraction of mis-clustered nodes using the $K$-means algorithm to the various parameters in the SBM.

In general, the $K$-means algorithm can be described as follows: Assume one wants to find $K$ clusters, for a given set of data points $x_i \in \mathbb{R}^K$, for $i = 1,\ldots, K$. Then the $K$-clusters  resulting from applying the $K$-means algorithm corresponds to a partition $\hatt = \{\hat{T}_1,\ldots,\hat{T}_K\}$ of   $\{1,\ldots,n\}$ that aims to minimize the following objective function over all such partitions:
\begin{equation}
\objec(\ttt) =\sum_{k=1}^K\sum_{i \in T_k} \|x_i - \bar{x}_{T_k}\|^2, \label{eq:objfun}
\end{equation}
Here $\mathcal{T} = \{T_1,\ldots,T_K\}$ is a partition  $\{1,\ldots,n\}$, and $\bar{x}_{T_k}$ corresponds to the vector of component-wise means of the $x_i$, for $i\in T_k$. 

In our situation there is also an underlying true partition of nodes into clusters, given by  $\mathcal{C} = \{C_1,\ldots,C_K\}$.  Notice that $\mathcal{C} = \hatt$ iff there is a permutation $\pi$ of $\{1,\ldots,K\}$ so that $C_{k} = \hat T_{\pi(k)}$, for $k = 1,\ldots,K$. 
In general, we use  the following measure to quantify the closeness of the outputted partition $\hatt$ and the true partition $\mathcal{C}$: Denote the \textit{clustering error} associated with $\hat{T}_1,\ldots, \hat{T}_K$ as 

\begin{equation}
\hat f = \min_\pi \max_k \frac{|C_k \cap \hat T_{\pi(k)}^c| + |C_k^c \cap \hat T_{\pi(k)}|}{\nn{k}} . \label{eq:clusterror}
\end{equation}
 
The clustering error measures the maximum  proportion of nodes in the symmetric difference of $C_k$ and $\hat T_{\pi(k)}$.

 In many situations, such as ours, there exists population quantities associated with each cluster around which the $x_i$'s are expected to concentrate. Denote these quantities by $m_1,\ldots,\, m_K$. In our case, $m_k = \cent_{k,\tau}$.  If the $x_i$'s, for $i \in C_k$, concentrate well around $m_k$, and the $m_k$'s are sufficiently well separated, then it is expected the $K$-means algorithm recovers the clusters with small error 	$\hat f$.  

Denote $\XX$ as the $n\times K$ matrix with $x_i$'s as rows. In our case, the $x_i = V_{i,\tau}$, and $X = V_\tau$. Further, denote as $\MM$ the $n\times K$ matrix with the $m_k$'s as rows. In our case, $\MM = \vsym_\tau$. Recent results on cluster recovery using the $K$-means algorithm, as given in \citet{kumar2010clustering} and \citet{awasthi2012improved}, provide  conditions on $\XX$ and $\MM$ for the success of $K$-means. The following lemma is implied from Theorem 3.1 in \citet{awasthi2012improved}.

\begin{lem} \label{lem:kumar} Let $\ddelta >0$ be a small quantity. If for each $1 \leq k \neq k' \leq K$, one has
\begin{equation}
\|m_k - m_{k'}\| \geq \left(\frac{1}{\ddelta}\right) \sqrt{K} \| \XX - \MM\|\left( {\frac{1}{\sqrt{\nn{k}}} + \frac{1}{\sqrt{\nn{k'}}}} \right)\label{eq:centersep}
\end{equation}
then the clustering error $\hat f = O\left(\ddelta^2\right)$ using the $K$-means algorithm. 
\end{lem}

  \textbf{Remark :} In general minimizing the objective function \eqref{eq:objfun} is not computationally feasible. However, the results in \cite{kumar2010clustering}, \cite{awasthi2012improved} can be extended to partitions $\hatt$ that approximately minimize \eqref{eq:objfun}.  The condition \eqref{eq:centersep}, called the \textit{center separation} condition in \cite{awasthi2012improved}, provides lower bounds on the pairwise distances between the population cluster centers that depend on the perturbation of data points around the population centers (represented by $\|X - M\|$) and the cluster sizes.

\vspace{.35cm}

Let $$1 = \mu_{1,\tau} \geq \ldots \geq \mu_{n,\tau} $$ 
be the eigenvalues of the regularized population Laplacian $ \ml_\tau$ arranged in decreasing order. 
The fact that $\mu_{1,\tau}$ is 1 follows from standard results on the spectrum of Laplacian matrices (see, for example, \cite{von2007tutorial}). As mentioned in the introduction, in order to control the perturbation of the first $K$ eigenvectors 
the  eigen gap, given by $\mu_{K,\tau} - \mu_{K+1,\tau}$, must be adequately large, as noted in \cite{von2007tutorial}, \cite{ng2002spectral}, \cite{kwok2013improved}. Since $\ml_\tau$  has rank $K$ one has $\mu_{K+1,\tau} = 0$. Thus the eigen gap is simply $\mu_{K,\tau}$. 
 For our $K$-block SBM framework the following is an immediate consequence of Lemma \ref{lem:kumar} and the Davis-Kahan theorem for the perturbation of eigenvectors.
\begin{cor} \label{cor:simcor} Let $\tau \geq 0$ be fixed. For the  RSC-$\tau$ algorithm the clustering error, given by \eqref{eq:clusterror}, is
$$O\left( \frac{K \| L_\tau - \ml_\tau\|^2 }{\mu_{K,\tau}^2} \right)$$
\end{cor}

\begin{proof}  Use  Lemma \ref{lem:kumar} with $m_k =\cent_{k,\tau}$, $\XX = V_\tau
$, $\MM = \vsym_\tau$, and notice that from Lemma \ref{lem:rescue} that $\|m_k - m_{k'}\|$ is  $\sqrt{1/\nn{k} + 1/\nn{k'}}$. 

Consequently, using $1/\sqrt{\nn{k}} + 1/\sqrt{\nn{k'}} \geq \sqrt{1/\nn{k} + 1/\nn{k'}}$ one gets  from \eqref{eq:centersep} that if
\begin{equation}
 \|V_\tau - \vsym_\tau\| \leq \frac{\ddelta}{\sqrt{K}}, \label{eq:fkmeg}
\end{equation}
for some $\delta >0$, then at most $O(\ddelta^2)$ fraction of nodes are misclassified with the RSC-$\tau$ algorithm.

From the Davis-Kahan theorem \cite{bhatia1997matrix}, one has
\begin{equation}
\|V_\tau - \vsym_\tau\| \lesssim\frac{\|L_\tau - \ml_\tau\|}{\mu_{K,\tau}} \label{eq:dkimp}
\end{equation}
Consequently, if we take $\ddelta = (\sqrt{K}\|L_\tau - \ml_\tau\|)/\mu_{K,\tau}$ then relation \eqref{eq:fkmeg} is satisfied using  \eqref{eq:dkimp}.  This proves the corollary.
\end{proof}

 
\section{Improvements through regularization} \label{sec:examples}

In this section we will use Corollary \ref{cor:simcor} to quantify improvements in clustering performance via regularization. 
If the number of clusters $K$ is fixed (does not grow with $n$) then the quantity
 \begin{equation}
 \frac{\|L_\tau - \ml_\tau\|}{\mu_{K,\tau}}, \label{eq:daviskahanbound}
 \end{equation}
  in Corollary \ref{cor:simcor} provides an insight into the role of the regularization parameter $\tau$. Clearly, an ideal choice of $\tau$ would be the one that minimizes \eqref{eq:daviskahanbound}. Note, however, that this is not practically possible since $\ml_\tau,\, \mu_{K,\tau}$ are not known in advance.

 
  Increasing $\tau$ will ensure that the Laplacian $L_\tau$ will be well concentrated around $\ml_\tau$. This is demonstrated in Theorem \ref{thm:someassump} below. 
However, increasing $\tau$ also has the effect of decreasing the eigen gap, which in this case is $\mu_{K,\tau}$, since the population Laplacian becomes more like a constant matrix upon increasing $\tau$.  Thus the optimum $\tau$  results from the balancing out of these two competing effects. 
 
Independent of our work, a similar argument for the optimum choice of regularization, using the Davis-Kahan theorem, was given in \citet{qin2013regularized} for the regulariztion proposed in  \cite{chaudhuri2012spectral}. However, they didn't provide a quantification of the benefit of regularization as given in this section and Section \ref{sec:selhigh}.




Theorem \ref{thm:someassump} provides high-probability bounds on the quantity $\| L_\tau - \ml_\tau\|$ appearing in the numerator of \eqref{eq:daviskahanbound}.  Previous analysis of the regularization \eqref{eq:chaudhari}, in \cite{chaudhuri2012spectral}, \cite{qin2013regularized}, show high-probability bounds on the aforementioned spectral norm that have a  $1/\sqrt{\mind + \tau}$ dependence on $\tau$. However, for large $\tau$,  the theorem below shows that the behavior is $\sqrt{\maxd}/(\maxd + \tau)$. We believe this holds for the regularization \eqref{eq:chaudhari} as well. Thus, our  bounds has a $1/\tau$ dependence on $\tau$, for large $\tau$, as opposed to the $1/\sqrt{\tau}$ dependence shown in \cite{chaudhuri2012spectral}. This is crucial since the eigen gap $\mu_{K,\tau}$ also behaves like $1/\tau$ for large $\tau$ which implies that  \eqref{eq:daviskahanbound} converges to a quantity as $\tau$ tends to infinity. In Theorem \ref{thm:kblockgenthm} we provide a bound on this quantity. Our claims regarding improvements via regularization will then follow from comparing this bound with the bound on \eqref{eq:daviskahanbound} at $\tau = 0$.  


\begin{theorem} \label{thm:someassump}
 With probability at least $1 - 2/n$,  for all $\tau$ satisfying
\begin{equation}
\max\{\tau, \mind\} \geq 32\log n, \label{eq:tausat}
\end{equation}
we have
 \begin{equation}
 \| L_\tau - \ml_\tau\| \leq \epst. \label{eq:pbdd}
 \end{equation}
  Here 
$$\epst = \begin{cases} \frac{10\sqrt{\log n}}{\sqrt{\mind +\tau}}, & \mbox{if } \tau \leq 2\maxd \\
&\\
 \frac{10\sqrt{\maxd\,\log n}}{\maxd + \tau/2}, & \mbox{if } \tau > 2\maxd \end{cases}$$
\end{theorem}


We use Theorem \ref{thm:someassump}, along with Corollary \ref{cor:simcor}, to demonstrate  improvements from regularization over previous analyses of eigenvector perturbation. Our strategy for  this is a follows: Take
$$\ddeltn = \frac{\epst}{\mu_{K,\tau}}$$
 Notice that from Corollary \ref{cor:simcor} and Theorem \ref{thm:someassump}, one gets that with probability at least $ 1- 2/n$, for all $\tau$ satisfying \eqref{eq:tausat}, the clustering error is $O(\ddeltn^2)$.
Consequently, it is of interest to study the quantity $\ddeltn$ as a function of $\tau$.  Define,
\begin{equation}
\ddeln = \lim_{\tau \rightarrow \infty} \ddeltn \label{eq:asymfrac}.
\end{equation}
 
Although we would have ideally liked to study the quantity, 
$$\tddeln = \min_{\max\{\tau,\, \mind\} \gtrsim \log n} \ddeltn$$
 we study $\ddeln$ since it is easy to characterize as we shall see in Theorem \ref{thm:kblockgenthm} below. Section \ref{sec:simresults} introduces a data-driven methodology that is based on finding an approximation for $\tddeln$. 


Before introducing our main theorem quantifying the performance of RSC-$\tau$ for large $\tau$ we introduce the follow definition.  

\begin{definition} Let $\{\tau_n,\, n \geq 1\}$ be a sequence of the regularization parameters. For the $K$-block SBM we say that RSC-$\tau_n$ gives consistent cluster estimates if the error \eqref{eq:clusterror} goes 0, with probability tending to 1, as $n$ goes to infinity.
\end{definition}


Throughout the remainder of the section we consider a $K$-block stochastic block model with the following block probability matrix.
\begin{equation}
\mb = \left( \begin{array}{cccc}
p_{1,n} & q_n & ... & q_n  \\
q_n & p_{2,n}  & ... & q_n\\
...   & ...& ...&...\\
... & ...& q_n & p_{K,n}
\end{array} \right). \label{eq:bdiag}
\end{equation}
The number of communities $K$ is assumed to be fixed. Without loss, assume that $p_{1,n} \geq p_{2,n} \ldots \geq p_{K,n}$. We also assume that $q_n < p_{K,n}$.  Denote $w_k = \nn{k}/n$, for $k = 1,\ldots, K$. The quantity $w_k$ represents the proportion of nodes belonging to the $k$-th community. Throughout this section we assume that $\{\tau_n : n\geq 1\}$ is a sequence of regularization parameters satisfying,
\begin{equation}
\frac{\left(\sum_{k=1}^K1/w_k\right)\maxd\log n}{\tau_n} = o(1) \label{eq:tausat1}
\end{equation}
Notice that if the cluster sizes are of the same order, that is $w_k \asymp 1$, then the above condition simply states that $\tau_n$ should grow faster than $\maxd\log n$.

Denote $\cldeg{k} = \nn{k} (p_{k,n} - q_n)$. The following is our main result regarding the impact of regularization.


\begin{theorem} \label{thm:kblockgenthm} For the $K$ block SBM, with block probability matrix \eqref{eq:bdiag},
\begin{equation}
\ddeln \asymp \frac{\left(\tmone \mone - \mtwo\right)}{\mone}\sqrt{\maxd\, \log n}. \label{eq:asympfrac}
\end{equation}
Here $\ddeln$ is given by \eqref{eq:asymfrac} and 
\begin{align}
m_{1,n} & = \sum_{k = 1}^K \frac{w_k }{\cldeg{k}} \label{eq:m1}\\
\tilde{m}_{1,n} & = \sum_{k = 1}^K \frac{1}{\cldeg{k}} \label{eq:tm1}\\
m_{2,n} & = \sum_{k = 1}^K \frac{w_k}{\cldeg{k}^2} \label{eq:m2}
\end{align}
Further, let $\{\tau_n,\, n \geq 1\}$ satisfy \eqref{eq:tausat1}. If  $\ddelta_n$ goes to 0, as $n$ tends to infinity, then  RSC-$\tau_n$  gives consistent cluster estimates.
\end{theorem}

Theorem \ref{thm:kblockgenthm} will be proved in Appendix \ref{sec:proofkblockgenthm}. In particular, the following corollary shows that for the stochastic block model regularized spectral clustering would work even when the minimum degree is of constant order. This is an improvement over recent works on unregularized spectral clustering, such as \cite{mcsherry2001spectral}, \cite{chaudhuri2012spectral}, \cite{rohe2011spectral},  which required the minimum degree to grow at least as fast as $\log n$.    

\begin{cor} \label{cor:twoblock} Let the block probability matrix  $\mb$ be as in \eqref{eq:bdiag}. 
Let $\{\tau_n,\, n\geq 1\}$ satisfy \eqref{eq:tausat1}. Then RSC-$\tau_n$ gives consistent cluster estimates under the following scenarios:
\begin{enumerate}[i)]
\item For the $K$-block SBM if  $w_k \asymp 1$, for each $k = 1,\ldots,\, K$, and 
\begin{equation}
\frac{(p_{K - 1,n} -q_n)^2}{p_{1,n}} \quad\mbox{grows faster than}\quad \frac{\log n}{n}.   \label{eq:fracmistwoblck}
\end{equation}
\item  For the 2-block SBM if $p_2 = q$ and
\begin{equation}
\frac{(p_{1,n} -q_n)^2}{w_1 p_{1,n} + w_2 q_n} \quad\mbox{grows faster than}\quad \frac{\log n}{n \left(\min\{w_1,\, w_2\}\right)^2}.   \label{eq:fracmistwoblckp}
\end{equation}
\end{enumerate}
\end{cor}

\textbf{Remark :} Regime $i)$ deals with the situation that the clusters sizes are of the same order of magnitude. Regime $ii)$, where $p_{2,n} = q_n$  mimics a scenario where there is only one cluster. This is a generalization of the \textit{planted clique} problem where $p_{1,n} = 1$ and $p_{2,n} = q = 1/2$. For the planted clique problem \eqref{eq:fracmistwoblckp} translates to requiring that $\min\{w_1,\, w_2\}$ grow faster that $\sqrt{\log n}/\sqrt{n}$ for consistent cluster estimates, which is similar to results in \cite{mcsherry2001spectral}.

\vspace{.3cm}

%
\begin{figure}[ht]
 \centering
\subfigure[Unregularized ($\tau = 0 $)]{\includegraphics[width=2.2in]{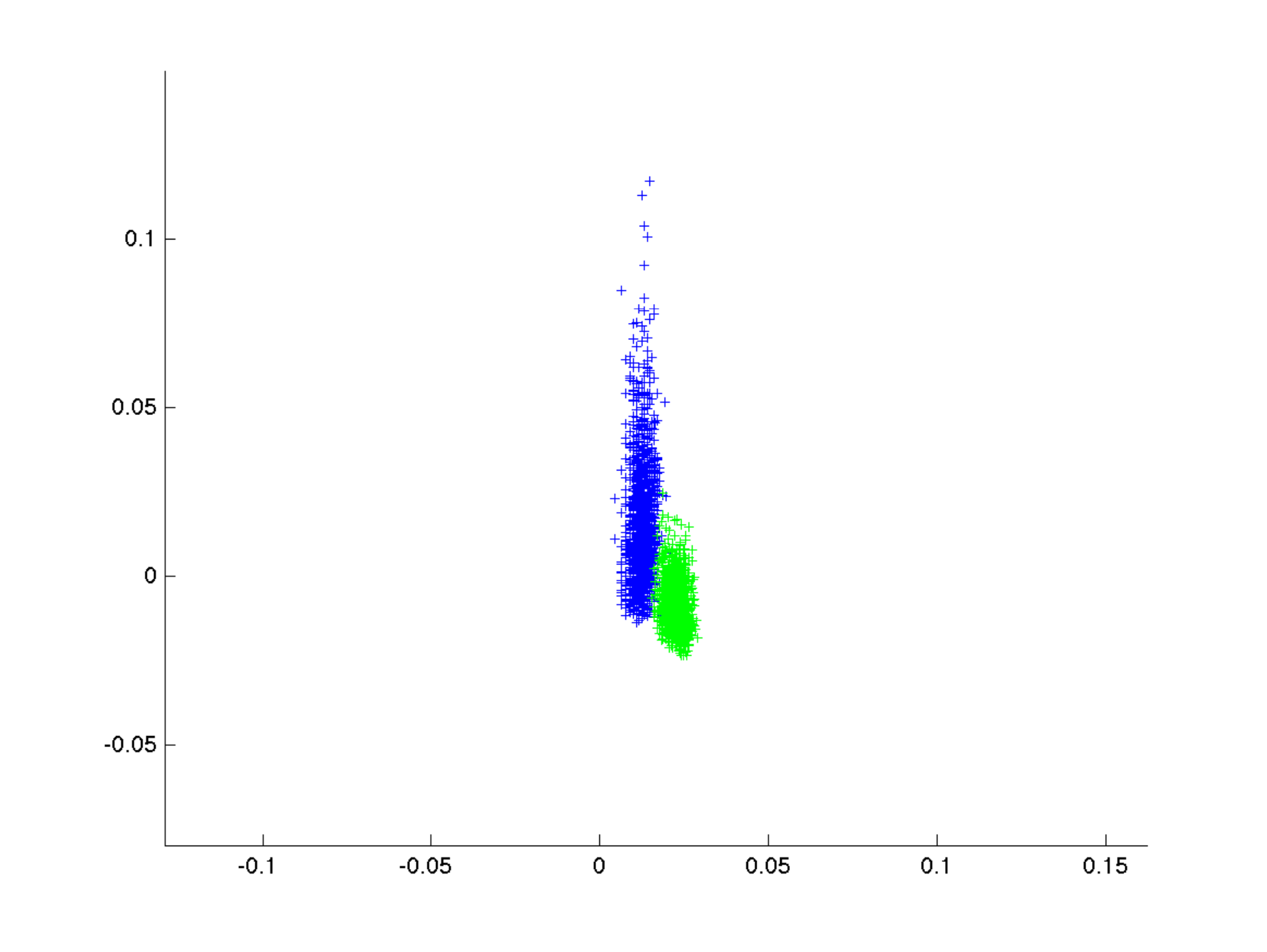}}
\subfigure[Regularized ($\tau = 26.5$)]{\includegraphics[width=2.2in]{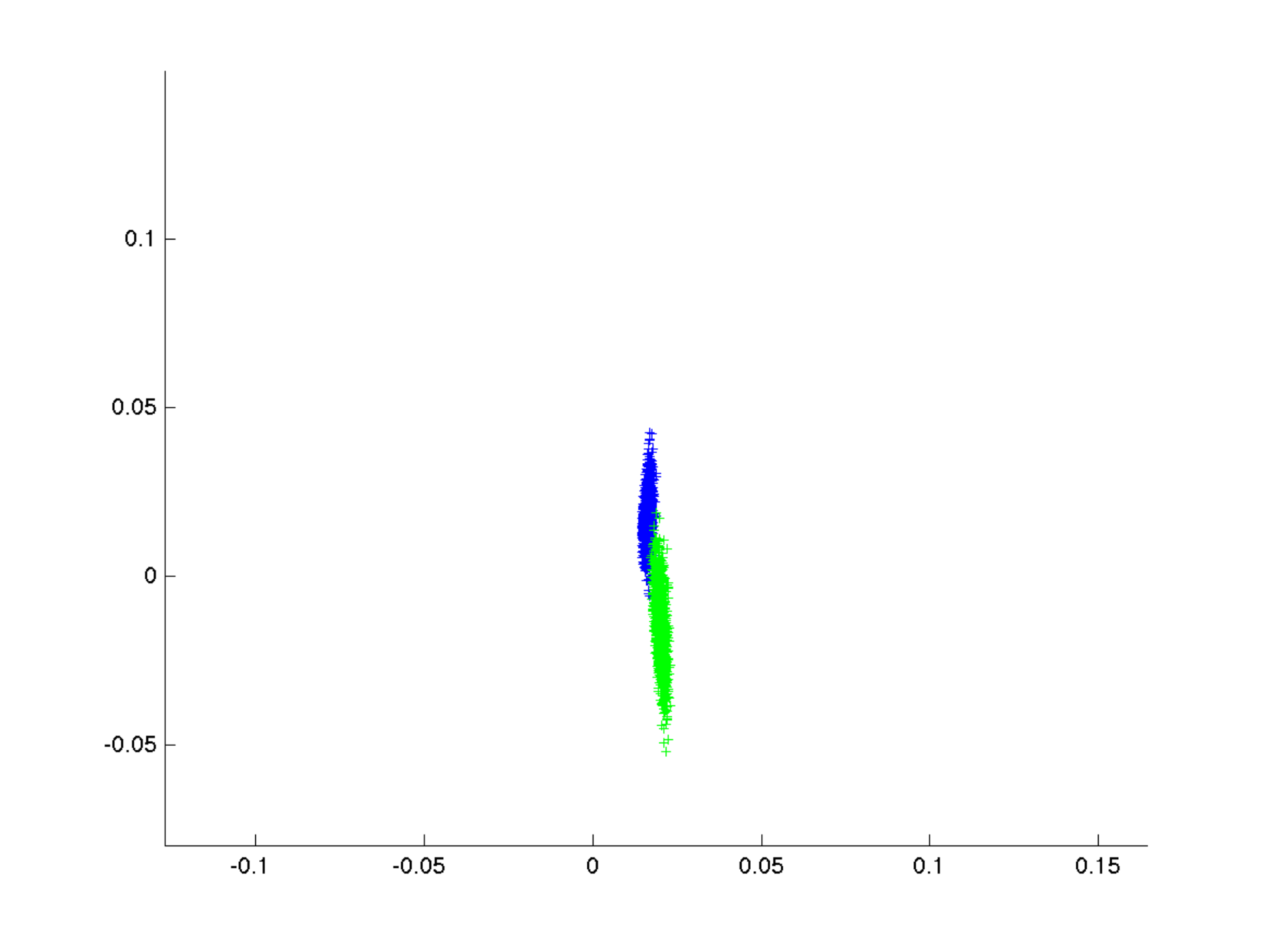}}
\subfigure[Regularized ($\tau = n$)]{\includegraphics[width=2.2in]{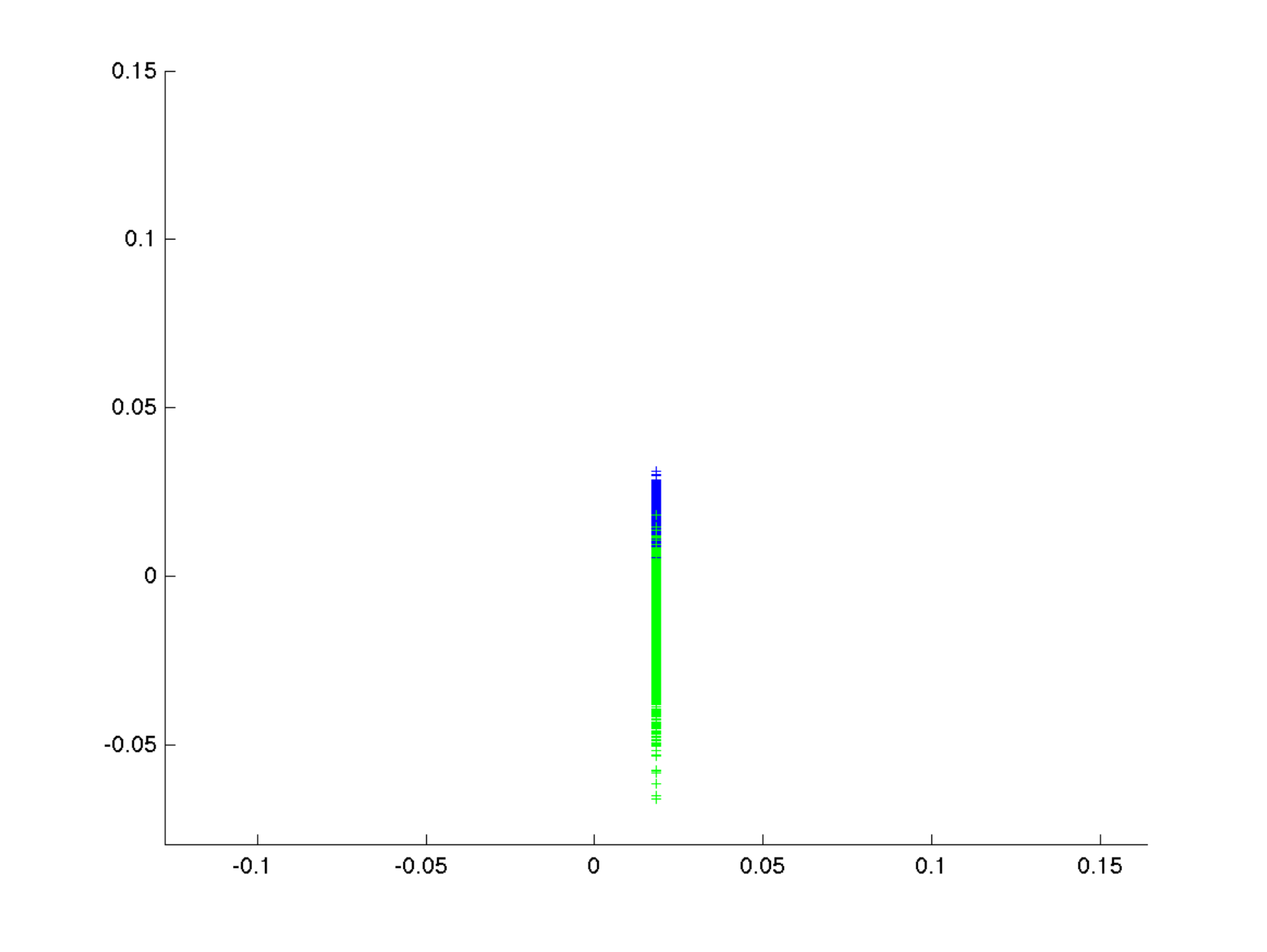}}

\caption[]{Scatter plot of first two eigenvectors with $\mb$ as in \eqref{eq:mbexample}. The $x,\, y$ axes provides values for the first, second eigenvectors respectively. The colors corresponds to the cluster memberships of the nodes.  Here the block probability matrix $\mb$ is as in \eqref{eq:mbexample}. Plot a) corresponds to $\tau = 0$.  b)  $\tau = 26.5$, selected using our data-driven $DKest$ methodology proposed in Section \ref{sec:simresults}.  c) $\tau = n$. 
}
\label{fig:plotcomp1}
\end{figure}

Notice that in both \eqref{eq:fracmistwoblck} and \eqref{eq:fracmistwoblckp} the minimum degree could be of constant order. For example, for the two-block SBM  if $q_n, p_{2,n} = O(1/n)$ then the minimum degree is of constant order. In this case ordinary spectral clustering using the normalized Laplacian would perform poorly. RSC performs better since from \eqref{eq:fracmistwoblck} it only requires that the larger of the two within block probabilities, that is $p_{1,n}$, growing appropriately fast. Figure \ref{fig:plotcomp1} illustrates this with $n = 3000$ and edge probability matrix
\begin{equation}
\mb =  \left( \begin{array}{cc}
.01 & .0025  \\
.0025 & .003	 \\
\end{array} \right). \label{eq:mbexample}
\end{equation}
 The figure provides the scatter plot of the first two eigenvectors of the unregularized and regularized sample Laplacians. Figure a) corresponds to the usual spectral clustering, while plots b) \& c) corresponds to RSC-$\tau$, with $\tau = 26.5,\, 3000$ respectively.  Here, $\tau = 26.5$ was selected using our data-driven methodology for slecting $\tau$ proposed in Section \ref{sec:simresults}. Also, $\tau = 3000$ was selected as suggested from Theorem \ref{thm:kblockgenthm} and Corollary \ref{cor:twoblock}. The fraction of mis-classified are $26\%,\, 4\%,\, 6\%$ for the cases a),\, b),\, c) respectively.

 
 
 From the scatter plots one sees that there is considerably less scattering for the blue points  with regularization. This results in improvements in clustering performance. Also, note that the performance in case c), in which $\tau$ is taken to be very large, is only slightly worse than case b). For case c) there is almost no variation in the first eigenvector, plotted along the $x$-axis. This makes sense since the first eigenvector is proportional to $(\sqrt{\hatdt{1}},\ldots,\, \sqrt{\hatdt{n}})$ and for large $\tau$ one has
 $\sqrt{\hatdt{i}} \approx \sqrt{\tau}$.


 It may seem surprising that in Corollary \ref{cor:twoblock}, claim \eqref{eq:fracmistwoblck}, the smallest within block probability, that is $p_{K,n}$ does not matter at all. One way of explaining this is that if one can do a good job identifying the top $K-1$ highest degree clusters then the cluster with the lowest degree can also be identified simply by eliminating nodes not belonging to this cluster. 

\section{SBM with strong and weak clusters} \label{sec:selhigh}

In many practical situations, not all nodes belong to clusters that can be estimated well. 
As mentioned in the introduction, these nodes interfere with the clustering of the remaining nodes in the sense that  none of the top eigenvectors might discriminate between the nodes that do belong to well-defined clusters. As an example of a real life data set, we consider the political blogs data set, which has two clusters,  in Subsection \ref{subsec:realanal}. With ordinary spectral clustering, the top two eigenvectors do not discriminate between the two clusters (see Figure \ref{fig:polblogprob} for explanation). Infact, it is only the third eigenvector that discriminates between the two clusters. This results in bad clustering performance when the first two eigenvectors are considered.  However,  regularization rectifies this problem by `bringing up' the important eigenvector thereby allowing for much better performance.

\begin{figure}[ht!]
\begin{center}$
\begin{array}{c}
\includegraphics[width=2.6in]{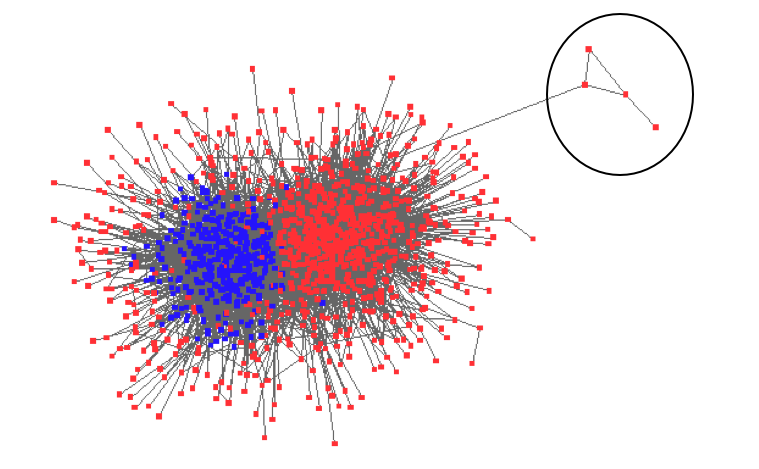}
\end{array}$
\end{center}
\caption[]{Depiction of the political blog network \cite{adamic2005political}. Instead of discriminating between the red and blue nodes, the second eigenvector discriminates the small cluster of 4 nodes (circled) from the remaining. This results in bad clustering performance.
 }
\label{fig:polblogprob}
\end{figure}




 We model the above situation -- where there are main clusters as well as outlier nodes --  in the following way: Consider a stochastic block model,  as in \eqref{eq:bdiag}, with $K + \kw$ blocks. In particular, let the block probability matrix be given by
\begin{equation}
 \mb =  \left( \begin{array}{cc}
B_s & B_{sw} \\
B_{sw}' & B_{w} \\
\end{array} \right),
 \label{eq:bdiagbig}
\end{equation}
where $B_s$ is a $K\times K$ matrix with $(p_{1,n}, \ldots, p_{K,n})$ in the diagonal and $q_{n}$ in the off-diagonal. Further, $B_{sw},\, B_w$ are $K\times \kw$ and $\kw\times \kw$ dimensional matrices respectively.  In the above $(K + \kw)$-block SBM, the top $K$ blocks corresponds to the well-defined or \textit{strong} clusters, while the bottom $\kw$ blocks corresponds to less well-defined or \textit{weak} clusters.  
 
 We now formalize our notion of strong and weak clusters.   The matrix $B_s$ models the distribution of edges between the nodes belonging to the strong clusters, while the matrix $B_w$ has the corresponding role for the weak clusters. The matrix $B_{sw}$ models the interaction between the strong and weak clusters.
   For ease of analysis, we make the following simplifying  assumptions : Assume that $p_{k,n} = \pkn$, for $k = 1,\ldots K$, and that the strong clusters $C_1,\, \ldots,\, C_K$ have equal sizes, that is, assume $\nn{k} = \nk$ for $k = 1,\ldots, K$.  

 Let $b_{sw} $  be defined as the maximum of the elements in $B_{sw}$,  
 and let $n^w$ be the number of nodes belonging to a weak cluster. In other words, $K\nk + n^w = n$. We make the following three assumptions:
 \begin{equation}
\frac{(\pkn - q_n)^2}{\pkn} \quad\mbox{grows faster than}\quad\frac{\log n}{n} \label{eq:absence}
\end{equation}
\begin{equation}
n^w = O(1) . \label{eq:smallvsweak}
\end{equation} 
\begin{equation}
 b_{sw} \lesssim \sqrt{\frac{\pkn\log n }{n}}\label{eq:smallvsweak2}
\end{equation}

Assumption \eqref{eq:absence}  ensures recovery of the strong clusters if there were no nodes belonging to weak clusters (See Corollary \ref{cor:twoblock} or \citet{mcsherry2001spectral}, Corollary 1). Assumption \eqref{eq:smallvsweak} and \eqref{eq:smallvsweak2} pertain to the nodes in the weak clusters. In particular, 
 Assumption \eqref{eq:smallvsweak} simply states that the total number of nodes belonging to a weak cluster is constant and does not grow with $n$. Assumption \eqref{eq:smallvsweak2} states that the density of the edges between the strong and weak clusters, denoted by $b_{sw}$, is not too large.

We only assume that the rank of $B_s$ is $K$. Thus,  the rank of $B$ is at least $K$. 
 As before, we assume that $K$ is known and does not grow with $n$. The number of weak clusters,  $\kw$,  need not be known and and could be as high as $n^w$. 
  We do not even place any restriction on the sizes of a weak cluster. Indeed, we even entertain the case that each of the $\kw$ clusters has one node. Consequently, we are only interested in recovering the strong clusters.
  
  
  Theorem \ref{thm:highdegcor1} presents our theorem for the recovery of the $K$ strong clusters using the RSC-$\tau_n$ Algorithm, with $\{\tau_n,\, n \geq 1\}$, satisfying
\begin{equation}
\frac{n\pkn\log n}{\tau_n} = o(1) \label{eq:tausat3}
\end{equation}  
In other words, the regularization parameter is taken to grow faster than $n\pkn \log n$, where notice that $n\pkn$ is of the same order of the expected maximum degree    of the graph.  Let $\hat T_1,\ldots, \hat T_K$ be the clusters outputted from the RSC-$\tau_n$ Algorithm. Let
  \begin{equation*}
\hat f = \min_\pi \max_k \frac{|C_k \cap \hat T_{\pi(k)}^c| + |C_k^c \cap \hat T_{\pi(k)}|}{\nn{k}} , 
\end{equation*}
be as in \eqref{eq:clusterror}. Notice that the clusters $C_1,\ldots, C_K$ do not form a partition of $\{1,\ldots,\, n\}$, while the estimates $\hat{T}_1,\, \ldots,\, \hat{T}_K$ do. However, since $n^w$ does not grow with $n$ this should not make much of a difference.

\begin{theorem} \label{thm:highdegcor1} Let Assumptions \eqref{eq:absence}, \eqref{eq:smallvsweak} and \eqref{eq:smallvsweak2} be satisfied.  If $\{\tau_n ,\, n\geq 1\}$ satisfies \eqref{eq:tausat3} then 	
the clustering error $\hat f$  for RSC-$\tau_n$ goes to zero with probability tending to one.
\end{theorem}

 The theorem is proved in Appendix \ref{sec:thmhighdeg}. It  states that under  Assumption \eqref{eq:absence} --  \eqref{eq:smallvsweak2} one can can get the same results with regularization that one would get if the nodes belonging to the weak clusters weren't present.
 
  Spectral clustering (with $\tau = 0$) may fail  under the above assumptions. This is elucidated in Figure \ref{fig:popimp}. Here $n = 2000$ and there are two strong clusters ($K=2$) and three weak clusters ($M=3$). The first 1600 nodes are evenly split between the two strong clusters, with the remaining nodes split evenly between the weak clusters. The matrix $B_s$ and $B_w$ are as in \eqref{eq:swmats} and $B_{sw}$ is a matrix with all entries $.015$.
\begin{equation}
 \mb_{s} =  \left( \begin{array}{cc}
.025 & .015 \\
.015 & .025\\
\end{array} \right)
\quad\quad\quad
  \mb_{w} =  \left( \begin{array}{ccc}
.007 & .015 & .015 \\
.015 & .0071 & .015\\
.015 & .015 & .0069\\
\end{array} \right). \label{eq:swmats}
\end{equation}
  The nodes in the weak clusters have relatively lower degrees, and consequently, cannot be recovered. 
  Figures  \ref{subfig:popunreg} and \ref{subfig:popreg} show the first 3 eigenvectors of the population Laplacian in the regularized and unregularized cases. We plot the first 3 instead of the first 5 eigenvectors in order to facilitate understanding of the plot.  In both cases the first eigenvector is not able to distinguish between the two strong clusters. This makes sense since the first eigenvector of the Laplacian has elements whose magnitude is proportional to square root of the population degrees (see, for example, \cite{von2007tutorial} for a proof of this fact). Consequently, as the population degrees are the same for the two strong clusters, the values for this eigenvector is constant for nodes belonging to the strong clusters.

 The situation is different for the second population eigenvector. In the regularized case,  the second eigenvector is able to distinguish between these two clusters. However, this is not the case for the unregularized case. From Figure \ref{subfig:popunreg}, not even the third unregularized eigenvector is able to distinguish between the strong and weak clusters.   Indeed, it is only the fifth eigenvector that distinguishes between the two strong clusters in the unregularized case.


\begin{figure}[t!] 
 \centering
\subfigure[Unregularized]{\includegraphics[width=2in]{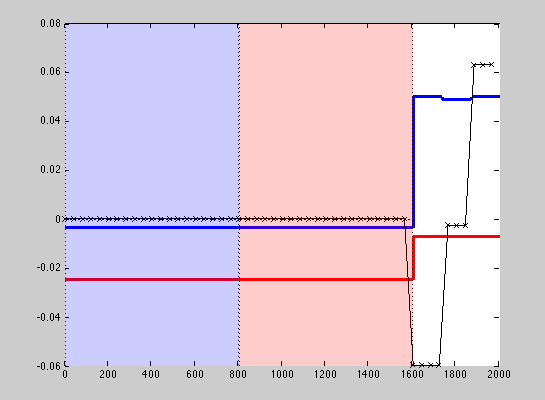} \label{subfig:popunreg}}
\quad
\subfigure[Regularized]{\includegraphics[width=2in]{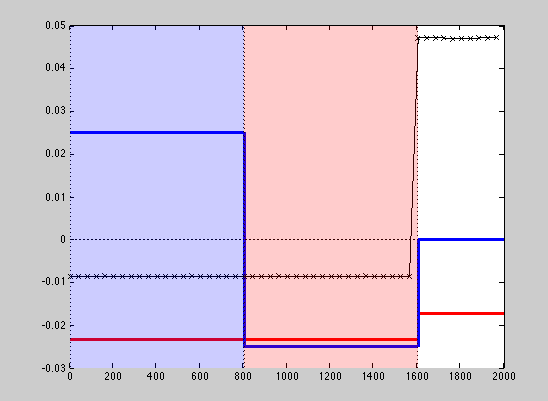} \label{subfig:popreg}}
\caption{First three population eigenvectors corresponding to $B_s$ and $B_{w}$ in \eqref{eq:swmats}.  In both plots, the x-axis provides the node indices while the y-axis gives the eigenvector values. The regularization parameter was taken to be $n$.  The shaded blue and pink regions corresponds to the nodes belonging to the two strong clusters.  The solid red line, solid blue line and $-\mathord{\times}-$ black lines correspond to the first, second and third population eigenvectors respectively.} \label{fig:popimp}
\end{figure}

\begin{figure}[ht!]
\centering
\subfigure[Unregularized]{\includegraphics[width=2in]{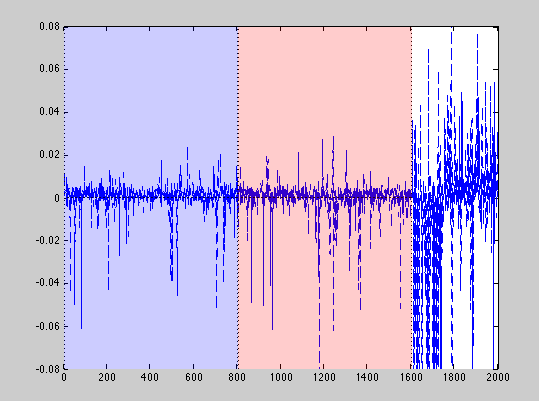}\label{subfig:sampunreg}}
\quad
\subfigure[Regularized]{\includegraphics[width=2in]{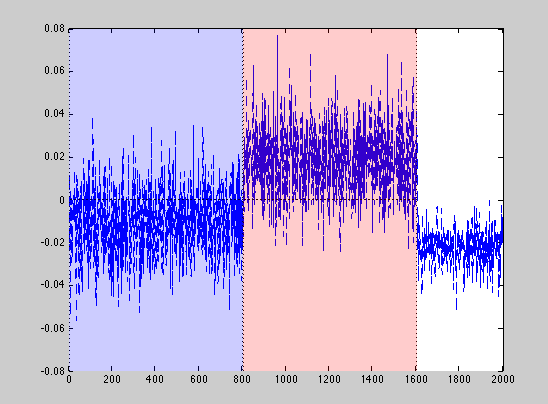}\label{subfig:sampreg}}

\caption[]{Second sample eigenvector corresponding to situation in  Figure \ref{fig:popimp}. As before, in both plots, the x-axis provides the node indices, while the y-axis gives the eigenvector values. As before, the shaded blue and pink regions corresponds to the nodes belonging to the two strong clusters. For plots (a) \& (b) the blue line correspond to the second eigenvector of the respective sample Laplacian matrices.
 }
\label{fig:on}
\end{figure}

In Figure \ref{subfig:sampunreg} and \ref{subfig:sampreg} we show the second sample eigenvector for the two cases in Figure \ref{subfig:popunreg} and \ref{subfig:popreg}. Note, we do not show the first sample eigenvector since from Figure \ref{subfig:popunreg} and \ref{subfig:popreg}, the corresponding population eigenvectors are not able to distinguish between the two strong clusters. As expected, it is only for the regularized case that one sees  that the second eigenvector is able to do a good job in separating the two strong clusters. Running $K$-means, with $k = 2$, resulted in a mis-classification of $49\%$ of the nodes in the strong clusters in the unregularized case, compared with $16.25\%$ in the regularized case.

\section{ $DKest$ : Data dependent choice of $\tau$} \label{sec:simresults}


The results Sections \ref{sec:examples} and   \ref{sec:selhigh} theoretically examined the gains from regularization for large values of regularization parameter $\tau$. Those results do not rule out the possibility that intermediate values of $\tau$ may lead to better clustering performance. In this section we propose a data dependent scheme to select the regularization parameter. We  compare it with the scheme in \cite{chenJSM} that uses the Girvan-Newman modularity \cite{bickel2009nonparametric}. 
We use the widely used normalized mutual information criterion (NMI) \cite{chen2012fitting}, \cite{yao2003information} to quantify the performance of the spectral clustering algorithm in terms of closeness of the estimated clusters to the true clusters. 

 Our scheme works by directly estimating the quantity in  \eqref{eq:daviskahanbound} in the following manner: For each $\tau$ in grid, an estimate $\hat{\ml}_\tau$ of $\ml_\tau$ is obtained using clusters outputted from the RSC-$\tau$ algorithm.  In particular, let $\hat C_{1, \tau},\, \ldots,\, \hat C_{K,\, \tau}$ be the estimates of the clusters $C_1,\ldots, C_K$ produced from running RSC-$\tau$.  The estimate $\hat{\ml}_\tau$ is taken as the population regularized Laplacian corresponding to 
 an estimated block probability matrix $\hat B$ and clusters $\hat C_{1,\, \tau},\, \ldots,\, \hat C_{K,\, \tau}$. More specifically,  the $(k_1,\, k_2)$-th entry of $\hat B$  is taken as
 \begin{equation}
   \hat B_{k_1,\, k_2} = \frac{\sum_{i \in \hat C_{k_1, \tau}, \,\,\,  j \in \hat C_{k_2, \tau}} A_{ij}}{|\hat C_{k_1,\tau}| | \hat C_{k_2,\tau }|} \label{eq:hatbprop}
 \end{equation}
 The above is simply the proportion of  edges between the nodes in the cluster estimates  $\hat C_{k_1, \tau}$ and $\hat C_{k_2,\tau}$.  The following statistic is then considered:
 \begin{equation}
\text{$DKest_\tau$} =  \frac{\|L_\tau - \hat{\ml}_\tau\|}{\mu_K\left(\hat{\ml}_\tau\right)}, \label{eq:dkeststatest}
 \end{equation}
where $\mu_K\left(\hat{\ml}_\tau\right)$ denotes the  the $K$-th smallest eigenvalue of $\hat{\ml}_\tau$. The $\tau$ that minimizes the $DKest_\tau$ criterion is then chosen.  Since this criterion provides an estimate of the Davis-Kahan bound, we call it the \textit{DKest} criterion.


We compare the above to the scheme that uses  Girvan-Newman modularity \cite{bickel2009nonparametric}, \cite{newman2004finding}, as suggested 
in \cite{chenJSM}. For a particular $\tau$ in the grid   the Girvan-Newman modularity is computed for the clusters outputted using the RSC-$\tau$ Algorithm. The $\tau$ that maximizes the modularity value over the grid is then chosen.


Notice that the best possible choice of $\tau$  would be the one that simply maximizes the NMI over the selected grid. However, this cannot be computed in practice since calculation of the NMI requires knowledge of the true clusters. Nevertheless, this provides a useful benchmark against which one can compare the other two schemes. We call this the `oracle' scheme.

\begin{figure*}[ht!]
\begin{center}$
\begin{array}{cc}
\includegraphics[width=2.4in, height = 2.2in]{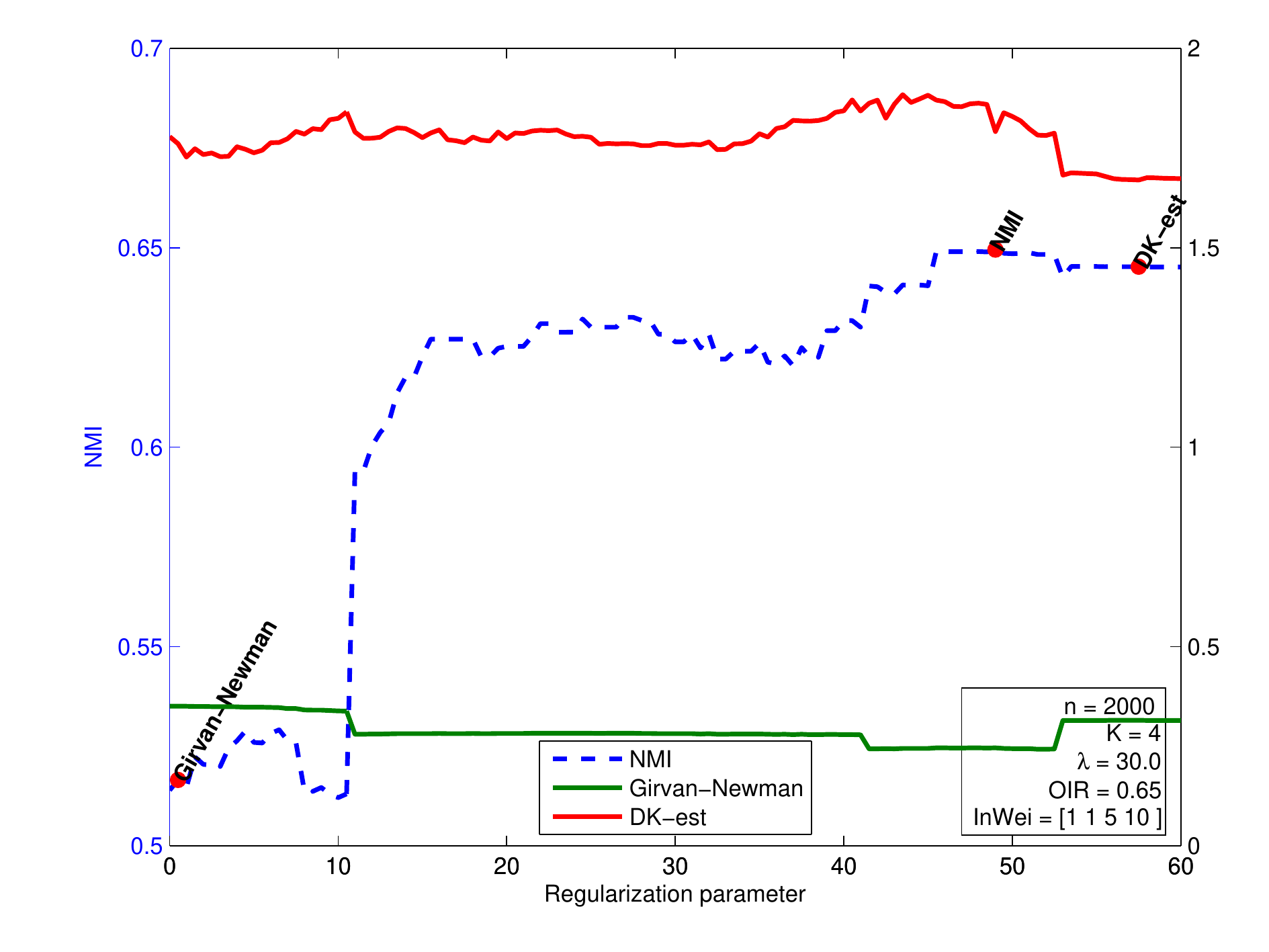}
&
\includegraphics[width=2.4in, height = 2.2in]{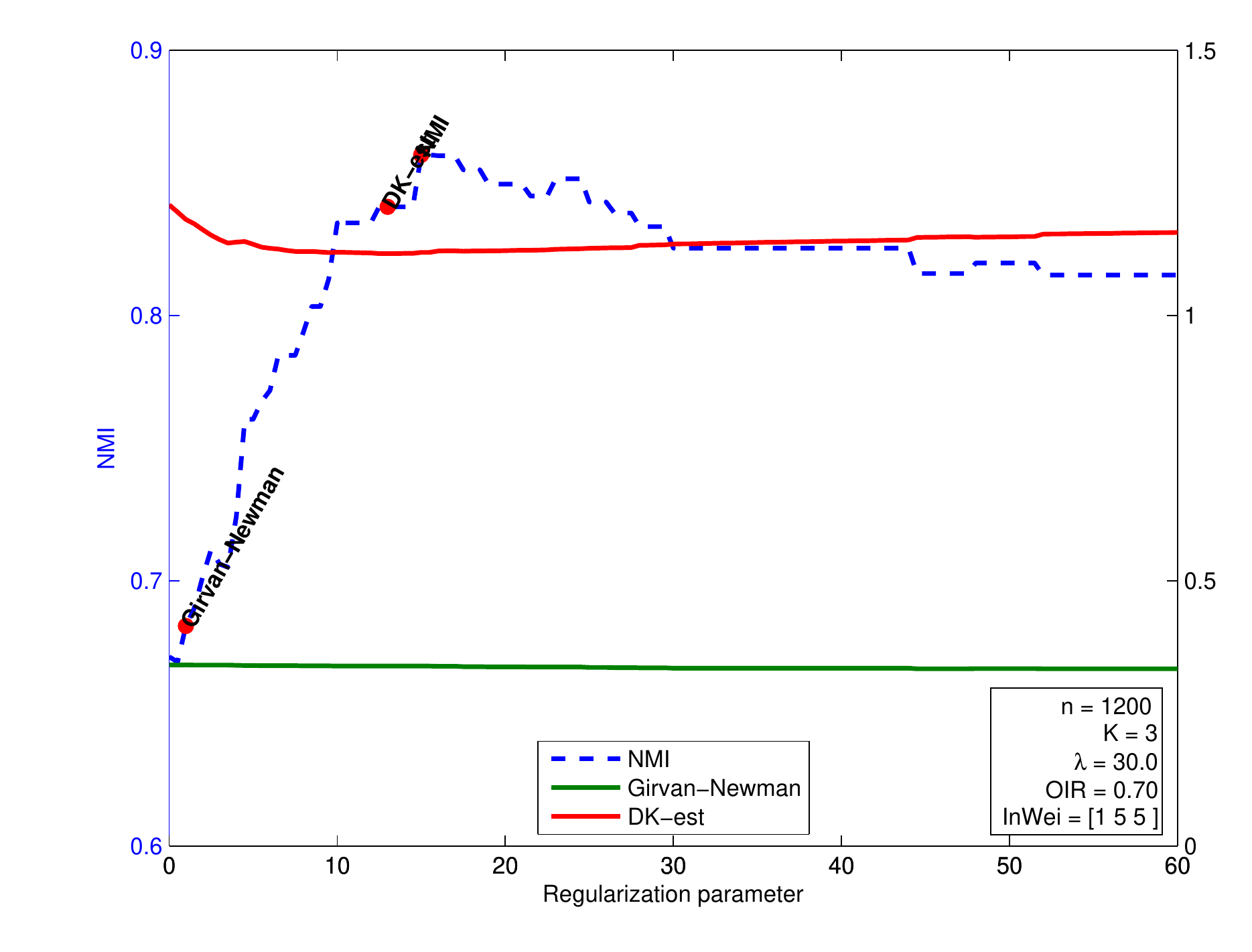}\\
\includegraphics[width=2.4in, height = 2.2in]{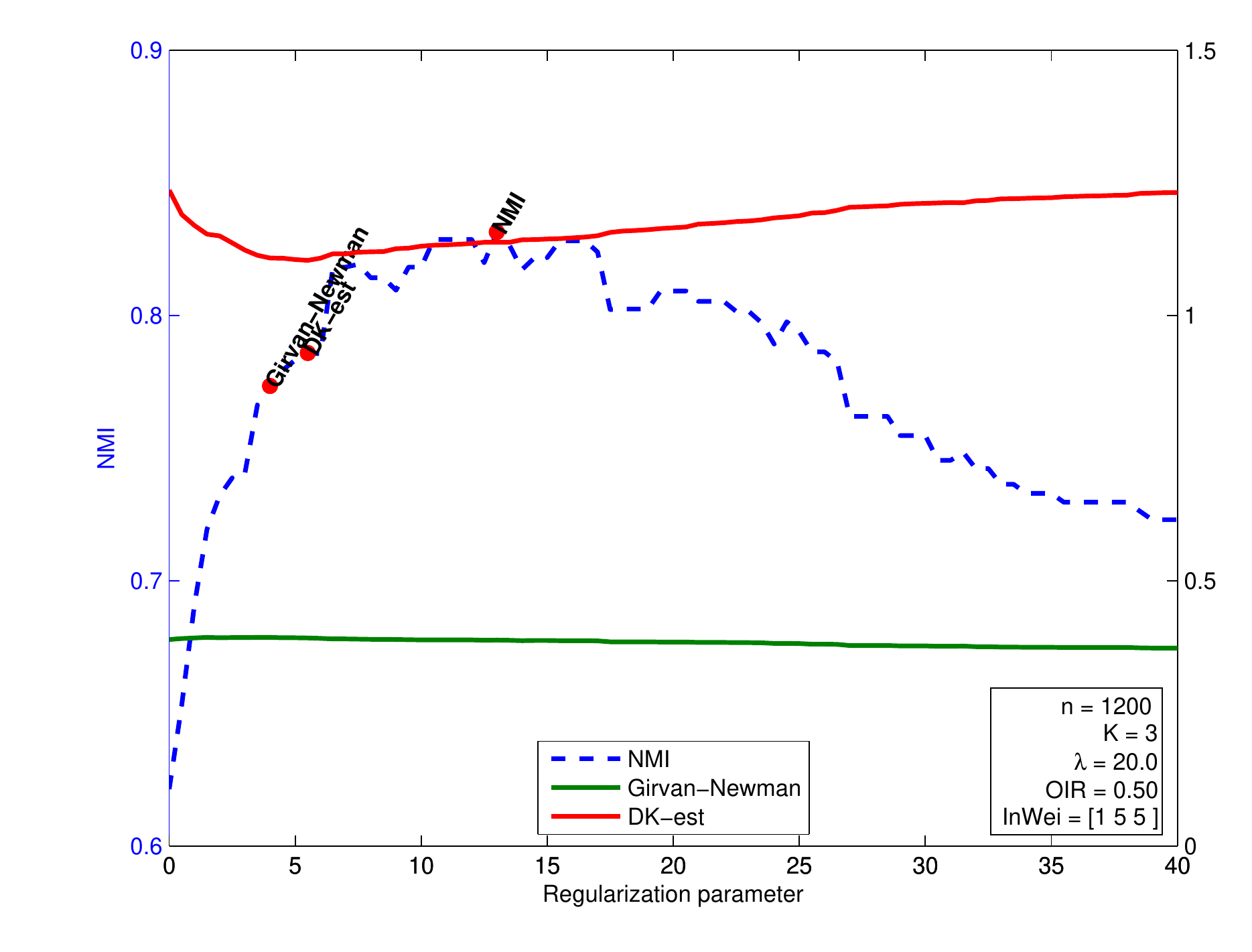}
&
\includegraphics[width=2.4in, height = 2.2in]{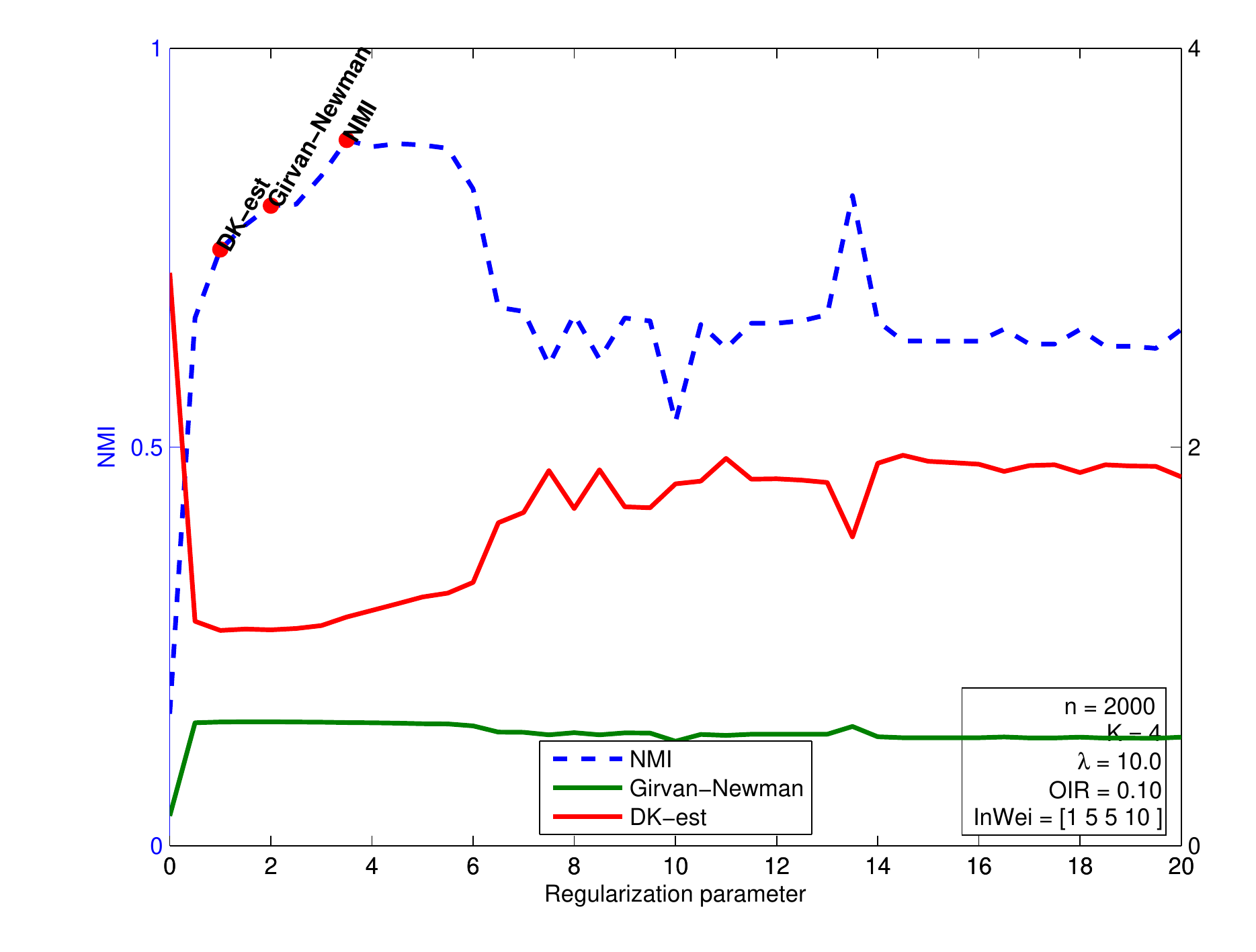}
\end{array}$
\end{center}
\caption[]{ Performance of spectral clustering as a function of $\tau$ for stochastic block model for $\lambda$ values of $30,\, 20$ and $10$. In the plots we denote $\beta$ and $w$ as $\textsf{OIR}$ and $\textsf{InWei}$ respectively.  The right $y$-axis provides values for the Girvan-Newman modularities and \textit{DKest} functions, while the left $y$-axis provides values for  the normalized mutual information (NMI). The 3 labeled dots correspond to values of the NMI at $\tau$ values which minimizes the \textit{DKest}, and maximizes the Girvan-Newman modularity and the NMI. Note,  the oracle $\tau$, or the $\tau$  that maximizes the NMI, cannot be calculated in practice. 
}
\label{fig:plotcomp2}
\end{figure*}

\subsection{Simulation Results} 
Figure \ref{fig:plotcomp2} provides results comparing the three schemes, viz. $DKest$, Girvan-Newman and `oracle' schemes. We perform simulations following the pattern of  \cite{chen2012fitting}. In particular, for a graph with $n$ nodes we take the $K$ clusters to be of equal sizes. The $K \times K$ block probability matrix is taken to be of the form
\begin{equation*}
\mb = \textsf{fac}\left( \begin{array}{cccc}
\beta w_1 & 1 & ... & 1  \\
1 & \beta w_2  & ... & 1\\
...   & ...& ...&...\\
... & ...& 1 & \beta w_K
\end{array} \right). 
\end{equation*}
Here, the vector $w = (w_1,\ldots, w_K)$, which are the \textit{inside weights}, denotes the relative degrees of nodes within the communities. Further, the quantity $\beta$, which is the \textit{out-in ratio}, represents the ratio of the probability of an edge between nodes from different communities to that of probability of edge between nodes in the same community.  The scalar parameter $\textsf{fac}$ is chosen so that the average expected degree of the graph is equal to $\lambda$. 

Figure \ref{fig:plotcomp2} compares the two methods of choosing the best $\tau$ for various choices of $n,\, K,\, \beta, \, w$ and $\lambda$. In general, we see that the $DKest$ selection procedure performs at least as well, and in some cases much better, than the procedure that used the Girvan-Newman modularity.  The performance of the two methods is much closer when the average degree is small.


\subsection{Analysis of the Political Blogs dataset} \label{subsec:realanal}

\begin{figure*}[ht!]
\begin{center}$
\begin{array}{c}
\includegraphics[width=3in]{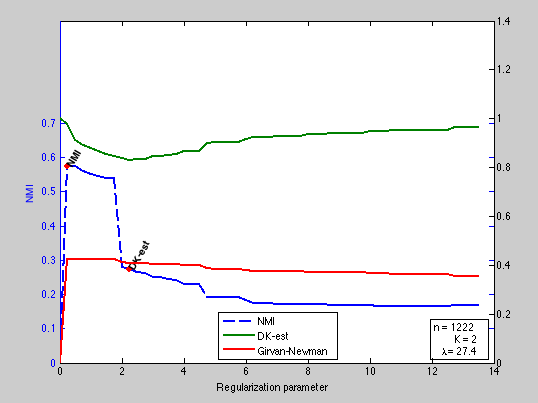}
\end{array}$
\end{center}
\caption[]{ Performance of the three schemes for the political blogs data set \cite{adamic2005political}.
}
\label{fig:polblog}
\end{figure*}

Here we investigate the performance of $DKest$ on the well studied network of political blogs \cite{adamic2005political}. The data set aims to study the degree of interaction between liberal and conservative blogs over a period prior to the 2004 U.S Presidential Election. The nodes in the networks are select conservative and liberal blog sites. While the original data set had directed edges corresponding to hyperlinks between the blog sites, we converted it to an undirected graph by connecting two nodes with an edge if there is at least one hyperlink from one node to the other.

 \begin{figure}[ht!]
\centering
\subfigure[Unregularized]{\includegraphics[width=2in]{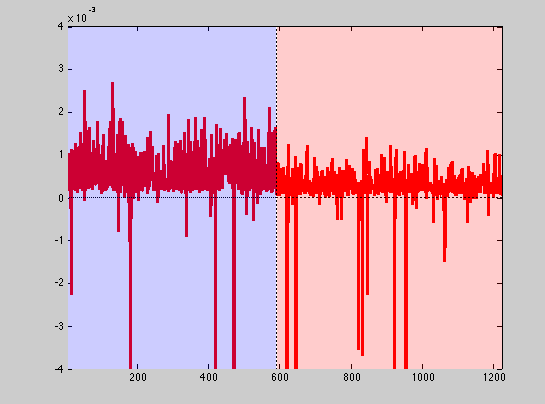}\label{subfig:polunreg}}
\quad
\subfigure[Regularized ($\tau = 2.25$)]{\includegraphics[width=2in]{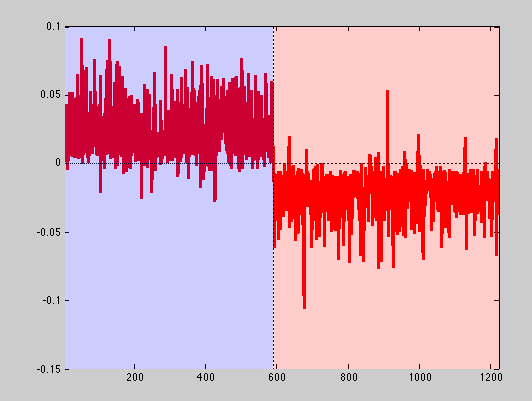}\label{subfig:polreg}}

\caption[]{Second eigenvector of the unregularized and regularized Laplacians for the political blogs data set \cite{adamic2005political}. The shaded blue and pink regions corresponds to the nodes belonging to the liberal and conservative blogs respectively. 
 }
\label{fig:pol}
\end{figure}

The data set has 1222 nodes with an average degree of 27. Spectral clustering ($\tau =0$) resulted in only 51\% of the nodes correctly classified as liberal or conservative.  The oracle procedure, with $\tau = 0.5$, resulted in 95\% of the nodes correctly classified. The \textit{DKest} procedure selected $\tau = 2.25$, with an accuracy of 81\%.  The Girvan-Newman (GN) procedure, in this case, outperforms the \textit{DKest} procedure providing the same accuracy as the oracle procedure. Figure \ref{fig:polblog} illustrates these findings.  As predicted by our theory, the performance becomes insensitive for large $\tau$. In this case 70\% of the nodes are correctly clustered for large $\tau$. 

We remark that the \textit{DKest} procedure does not perform as well as the GN procedure most likely because our estimate $\hat{\ml}_\tau$ in \eqref{eq:dkeststatest} assumes that the data is generated from an SBM, which is a poor model for the data due to the large heterogeneity in the node degrees. A better model for the data would be the degree corrected stochastic block model (D-SBM) proposed by \citet{karrer2011stochastic}. If we use D-SBM based estimaes in $DKest$ then the selection of $\tau$ matches that of the GN Newman and the oracle procedure. See Section \ref{sec:discuss} for a discussion on this.  

\begin{figure}[ht!]
\centering
\includegraphics[width=2in]{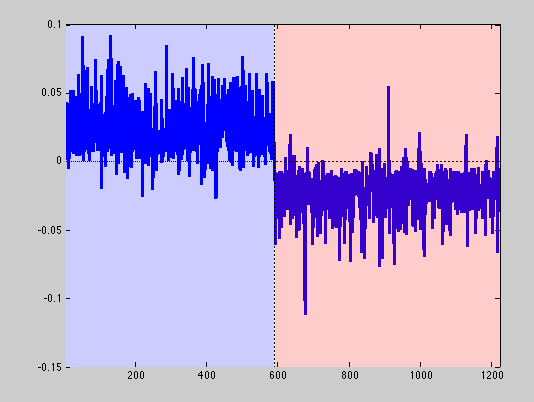}\label{subfig:thirdunreg}
\caption[]{Third eigenvector of the unregularized  Laplacian.
 }
\label{fig:polthird}
\end{figure}

The results of Section \ref{sec:selhigh} also explain why unregularized spectral clustering performs badly (see Figure \ref{fig:polblogprob}).  The first eigenvector in both cases (regularized and unregularized) does not discriminate between the two clusters. In Figure \ref{fig:pol}, we plot the second eigenvector of the regularized and unregularized Laplacians.  The second eigenvector is able to discriminate between the clusters in the regularized case, while it fails to do so in without regularization. Indeed, it is only the third eigenvector in the unregularized case that distinguishes between the clusters, as shown in Figure \ref{fig:polthird}.


\section{Discussion} \label{sec:discuss}


 The paper provides a theoretical justification for regularization. In particular, we show why choosing a large regularization parameter can lead to good results. The paper also partly explains  empirical findings in \citet{chen2012fitting} showing that the performance of regularized spectral clustering becomes insensitive  for larger values of regularization parameters.  It is unclear at this stage whether the benefits of regularization, resulting from the trade-offs between the eigen gap and the concentration bound, hold for the regularization in \cite{chaudhuri2012spectral}, \cite{qin2013regularized} as
 they hold for the regularization in \citet{chen2012fitting} (as  demonstrated in Sections \ref{sec:examples} and  \ref{sec:selhigh}).
 
 Even though our theoretical results focus on larger values of the regularization parameter it is very likely that intermediate values of $\tau$  produce better clustering performance. Consequently, we propose a data-driven methodology for choosing the regularization parameter. We hope to quantify theoretically the gains from using intermediate values of the regularization parameter in a future work.

 For the extension of the SBM proposed in Section \ref{sec:selhigh}, if the rank of $\mb$, given by \eqref{eq:bdiagbig}, is $K$ then the model encompasses specific degree-corrected stochastic block models (D-SBM) \cite{karrer2011stochastic} where  the edge probability matrix takes the form
$$P = \Theta Z \mb Z' \Theta.$$
 Here $\Theta = diag(\theta_1,\, \ldots,\, \theta_n)$ models the heterogeneity in the degrees. In particular, consider a $K$-block D-SBM with $0< \theta_i \leq 1$,  for each $i$.  Assume that $\theta_i =1$ for the most of the nodes. Take the nodes in the strong clusters to be those with $\theta_i = 1$. The nodes in the strong clusters are associated to one of $K$ clusters depending on the cluster they belong to in the D-SBM. The remaining nodes are taken to be in the weak clusters. Assumptions \eqref{eq:smallvsweak} and  \eqref{eq:smallvsweak2}  puts constraints on the $\theta_i$'s which allows one to distinguish between the strong clusters via regularization. It would be interesting to investigate  the effect of regularization in more general versions of the D-SBM, especially where there are high as well as low degree nodes. 
 
 The $DKest$ methodology for choosing the regularization parameter  works by providing  estimates of the population Laplacian assuming that the data is drawn from an SBM. From our simulations, it is seen that the performance of $DKest$ does not change much if we take the matrix norm in the numerator of \eqref{eq:dkeststatest} to be the Frobenius norm, which is much faster to compute. 
 
  It is seen that the performance of $DKest$ improves for the political blogs data set by taking $\hat \ml_\tau$ to be the estimate assuming that the data is drawn from the more flexible D-SBM.  Indeed, if we take $\hat \ml_\tau$ to be such an estimate  then the performance of $DKest$ is seen to be as good as the oracle scheme (and the GN scheme) for this data set. We describe how we construct this estimate in  Appendix \ref{sec:extdkest}.  

\section*{Acknowledgments}
This paper is supported in part by NSF grants DMS-1228246 and DMS-1160319 (FRG), ARO grant W911NF-11-1-0114, NHGRI grant 1U01HG007031-01 (ENCODE), and the Center of Science of Information (CSoI), a US NSF Science and Technology Center, under grant agreement CCF-0939370.  A. Joseph would like to thank Sivaraman Balakrishnan and Puramrita Sarkar for some very helpful discussions, and also Arash A. Amini for sharing the code used in the work \cite{chen2012fitting}.

\appendices


\section{Analysis of SBM with $K$ blocks} \label{sec:needtoverify}
Throughout this section we assume that we have samples from a $K$ block SBM. Denote the sample and population regularized Laplacian as $L_\tau,
\, \ml_\tau$ respectively. For ease of notation, we remove the subscript $\tau$ from the various matrices such as $L_\tau,\, \ml_\tau,\, A_\tau, D_\tau,\, \md_\tau$. We also remove the subscript $\tau$ in the  $\hatdt{i}, \sdt{i}$'s and denote these as $\hatd_i,\,\sd_i$ respectively.  However, in some situations we may need to refer to these quantities at $\tau=0$. In such cases, we make this clear by writing them as $\hatd_{i,0}$, for $i = 1,\ldots,n$ and $\sdgen_{i,0}$ for $i = 1,\ldots, n$.

We need probabilistic bounds on the weigthed sum of Bernoulli random variables. 
The following lemma is proved in  \cite{barron2010ajoseph}.

\begin{lem}\label{lem:bernoullisums} Let $W_j$, $1 \leq j \leq N$ be $N$  independent $\mbox{Bernoulli}(r_j)$ random variables. Furthermore, let $\alpha_j,\, 1\leq j \leq N$ be non-negative weights that sum to 1 and let $ N_\alpha = 1/\max_{j}\alpha_j$.
Then the weighted sum $\hat{r} = \sum_{j}\alpha_j W_j$, which has mean given by $r^* = \sum_{j} \alpha_j r_j$, satisfies the following large deviation inequalities.
For any $r$ with $0< r < r^*$,
\begin{equation}
P(\hat{r} < r ) \leq \exp\left\{-N_\alpha D(r\|r^*)\right\}
\end{equation}
and for any $\tilde{r}$ with  $r^* < \tilde{r} < 1$,
\begin{equation}
P(\hat{r} > \tilde{r} ) \leq \exp\left\{-N_\alpha D(\tilde{r}\|r^*)\right\}
\end{equation}
where $D(r\|r^*)$ denotes the relative entropy between Bernoulli random variables of success parameters $r$ and $r^*$.
\end{lem}

The following is an immediate corollary of the above.
\begin{cor} \label{cor:bddwj}
	Let $W_j$ be as in Lemma \ref{lem:bernoullisums}. Let $\beta_j$, for $j = 1,\ldots,N$ be non-negative weights, and let 
	$$W = \sum_{j= 1}^N \beta_j W_j.$$
Then,
\begin{equation}
P\left( W - E(W) > \delta \right) \leq \exp\left\{ -\frac{1}{2\max_j \beta_j } \frac{\delta^2}{(E(W) + \delta)}        \right\} \label{eq:wexpb1}
\end{equation}

and
\begin{equation}
P\left( W - E(W) < -\delta \right) \leq \exp\left\{ -\frac{1}{2\max_j \beta_j } \frac{\delta^2}{E(W)}  \right\}\label{eq:wexpbd2}
\end{equation}
 
\end{cor}

\begin{proof} Here we use the fact that 
\begin{equation}
D(r || r^*) \geq (r - r^*)^2/(2r), \label{eq:klbdd}
\end{equation}
 for any $0 < r,\, r^* < 1$.
We prove \eqref{eq:wexpb1}. The proof of \eqref{eq:wexpbd2} is similar. The event under consideration may be written as
$$\{ \hat r - r^* > \tilde{\delta}\},$$
where $\hat r = W/\sum_{j} \beta_j$, \, $r^* = E(W)/\sum_{j} \beta_j$ and $\tilde{\delta} = \delta/\sum_{j} \beta_j$. 
Correspondingly, using Lemma \ref{lem:bernoullisums} and \eqref{eq:klbdd}, one gets that 
$$P\left( W - E(W) > \delta \right)  \leq \exp\left\{ -\frac{\sum_{j} \beta_j}{\max_j \beta_j} \frac{\tilde \delta^2}{2 (r^* + \tilde \delta)}\right\}.$$
Substituting the values of $\tilde \delta$ and $r^*$ results in bound \eqref{eq:wexpb1}.
\end{proof}


The following lemma provides high probability bounds on the degree. 
Let $\taun = \max\{\mind,\, c\log n\}$ and $\delta_{i,c} = \max\{\sd_{i,0},\, c\log n\}$.

\begin{lem} \label{lem:highprobbdd} On a set $E_1$ of probability at most $1 - 2/n^{c_1 - 1}$, one has
$$|\hatdt{i}  - \sdt{i}|  \leq c_2 \sqrt{\delta_{i,c} \log n}  \quad \text{for each $i = 1,\ldots,n$}.,$$
where $c_1 = .5c_2^2/(1 + c_2/\sqrt{c})$.
\end{lem}
\begin{proof}
Use the fact that $\hatdt{i} - \sdt{i} = \hatd_{i,0} - \sd_{i,0}$, and  
$$P(|\hatd_{i,0} - \sd_{i,0}|  \leq c_2 \sqrt{\delta_{i,c} \log n} \quad \forall i ) \leq \sum_{i = 1}^nP(|\hatd_{i,0} - \sd_{i,0}| \leq c_2 \sqrt{\delta_{i,c} \log n})$$
Notice that $\hatd_{i,0} = \sum_{j = 1}^n A_{ij}$.  Apply Corollary \ref{cor:bddwj} with $\beta_j = 1$ and $W_j = A_{ij}$, and $\delta = c_2 \sqrt{\taun \log n}$ to bound each term in the sum of the right side of the above equation.

The error exponent can be bounded by,
\begin{equation}
2n\exp\left\{ -\frac{1}{2 } \frac{\delta^2}{(E(W) + \delta)}        \right\} . \label{eq:errorexpon}
\end{equation}
We claim that,
\begin{equation}
E(W) + \delta \leq (1 + c_2/\sqrt{c}) \delta_{i,c}. \label{eq:bddewdel}
\end{equation}
Substituting the above bound in the error exponent \eqref{eq:errorexpon} will complete the proof.

To see the claim, notice that $E(W) = \sd_{i,0}$. Now, consider the case $\sd_{i,0}\geq c\log n$. 
In this case, $\delta_{i,c} = \sd_{i,0}$ and $\log n < \sd_{i,0}/c$. Correspondingly, $E(W) + \delta$ is at most $\sd_{i,0}(1 + c_2/\sqrt{c})$.

Next, consider the case $\sd_{0,i} < c\log n$. In this case $\delta_{i,c} = \taun$,  which is $c \log n$.  Consequently, 
$$E(W) + \delta \leq c\log n + c_2 \sqrt{c} \log n.$$
The right side of the above  can be bounded by $(1 + c_2/\sqrt{c})(c\log n)$. This proves the claim.
\end{proof}


\subsection{Concentration of Laplacian} \label{subsec:conclap}
 
Below we provide the proof of Theorem \ref{thm:someassump}. Throughout this section we assume that the quantities $c,\, c_2$ appearing in Lemma \ref{lem:highprobbdd} are given by $c = 32$ and $c_2 = 2\sqrt{2}$. Notice that this makes $c_1 > 2$, where $c_1$ as in Lemma \ref{lem:highprobbdd}. 

From Lemma \ref{lem:highprobbdd}, with probability at least $1 - n^{-1}$,
$$\max_i| \hatd_i - \sd_i|/\sd_i \leq \max_i c_2 \sqrt{\delta_{i,c} \log n }/\sd_i$$
We claim that the right side of the above is at most $1/2$. To see this notice that
\begin{align*}
\sqrt{\delta_{i,c} \log n }/\sd_i &\leq \sqrt{\delta_{i,c} \log n }/\delta_{i,c}\\
&= \sqrt{\log n}/\sqrt{\delta_{i,c} }\\
&\leq 1/\sqrt{c}
\end{align*}
Here the first inequality follows from noting that $\sd_i = \sd_{i,0} + \tau$, which is at most $\max\{\sd_{i,0},c\log n\}$, using $\tau \geq c\log n$. The third inequality follows from using $\delta_{i,c} \geq c\log n$. Consequently, 
$\max_i|\hatd_i -\sd_i|/\sd_i \leq 1/2$ using $c_2 = 2\sqrt{2}$ and $c = 32$.

\begin{proof}[Proof of Theorem \ref{thm:someassump}]
Our proof has parallels with the proof in \cite{oliveira2009concentration}.  Write $\tilde L = \md^{-1/2} A \md^{-1/2}$. Then,
$$\|L - \ml\| \leq \|L - \tilde L\| + \|\tilde L - \ml\|.$$
We first bound $\|L -\tilde L\|$.  Let $F = D^{1/2}\md^{-1/2}$. Then $\tilde L = F L F$. Correspondingly,
\begin{align}
\| L - \tilde L\| &\leq \|L -FL \| +\|FL -\tilde L\|\nonumber\\
&\leq \|I - F\|\|L \| + \|F\|\|L \|\|I - F\|\nonumber\\
& \leq \|I - F\| \left(2 + \| I - F\|\right) \label{eq:implater}
\end{align}
Notice that 
$$F - I = (I +( D - \md) \md^{-1} )^{1/2} - I.$$
Further, using $\max_i| \hatd_i - \sd_i|/\sd_i   \leq 1/2$, and the fact that $\sqrt{1+x} - 1\leq x$ for $x\in[-3/4,3/4]$, as in \cite{oliveira2009concentration},  one gets that
$$\|F - I\| \leq c_2 \frac{\max_i \sqrt{\delta_{i,c} \log n }}{\sd_i}$$
with high probability. 
Consequently, using \eqref{eq:implater}, one gets that
\begin{equation}
\| L - \tilde L\| \leq c_2\max_i \frac{ \sqrt{\delta_{i,c} \log n }}{\sd_i}\left(2 + c_2 \max_i\frac{ \sqrt{\delta_{i,c} \log n }}{\sd_i}\right) \label{eq:firstconclap}
\end{equation}
with probability at least $1 - 1/n^{c_1 - 1}$. 
$$\max_i\frac{ \sqrt{\delta_{i,c} }}{\sd_i} \leq \tepst = \begin{cases} \frac{1}{\sqrt{\mind +\tau}}, & \mbox{if } \tau \leq 2\maxd \\
&\\
 \frac{\sqrt{\maxd}}{\maxd + \tau/2}, & \mbox{if } \tau > 2\maxd \end{cases}$$
 To see this notice, that $\delta_{i,c} \leq \sd_{i,0} + \tau = \sd_i$, using $\max\{\tau, \sd_{i,0}\} \geq c\log n$. Consequently , $\sqrt{\delta_{i,c} }/\sd_i \leq 1/\sqrt{\sd_{i,0} + \tau}$. which is at most $1/\sqrt{\mind + \tau}$.
 
 Further, $$\max_i\frac{ \sqrt{\delta_{i,c} }}{\sd_i}  \leq  \frac{\sqrt{\maxd}}{\maxd + \tau}$$ for $\tau > \maxd$. This is atmost $\tepst$ for $\tau > \maxd$.
 
Consequently, from \eqref{eq:firstconclap}, one gets that 
\begin{align}
\|L - \tilde L\| \leq c_2\,\tepst \sqrt{\log n}\left(2 + c_2/\sqrt{c}\right) 
\label{eq:secondfirstconclap}
\end{align}
with probability at least $1 - 1/n^{c_1 - 1}$.



Next, we bound $\|\tilde L - \ml\|$.  We get high probability bounds on this quantity using results in \cite{oliveira2009concentration}, \cite{mackey2012matrix}. In particular, as in \cite{oliveira2009concentration}, 
$$\tilde L - \ml = \sum_{i \leq j}  Y_{ij},$$
where $Y_{ij} = \md^{-1/2} X_{ij} \md^{-1/2} $, with
$$X_{ij} = \begin{cases} (A_{ij} - \mpp_{ij}) \left(e_i e_j^T + e_j e_i^T\right), & \mbox{if } i \neq j \\ (A_{ij} - \mpp_{ij}) e_i e_i^T  & \mbox{if } i = j \end{cases}.$$
Further, $\| Y_{ij}\| \leq 1/(\mind + \tau)$. Let $\sigma^2 = \|\sum_{i \leq j} E(Y_{ij}^2)\|$. We claim that $\sigma^2 \leq \tepst^2$. As in \cite{oliveira2009concentration}, page 15, notice that,
\begin{equation}
\sum_{i \leq j} E(Y_{ij}^2) = \sum_{i  = 1}^n \frac{1}{\sd_{i,0} + \tau} \left(\sum_{j = 1}^n \frac{P_{ij}(1 - P_{ij})}{\sd_{j,0} + \tau}\right)e_i e_i^T. \label{eq:oliveconcmain}
\end{equation}
 Clearly, 
\begin{align*}
\left(\sum_{j = 1}^n \frac{P_{ij}(1 - P_{ij})}{\sd_{j,0} + \tau}\right)  \leq \frac{\sd_{i,0}}{\mind + \tau}.
\end{align*}
Consequently, for each $i$ the right side of \eqref{eq:oliveconcmain} is at most $1/(\mind + \tau)$ leading to the fact that $\sigma^2 \leq 1/(\mind +\tau)$.

For $\tau > 2\maxd$ we can get improvements in the bound for $\sigma^2$. By using the fact that $\sd_{j,0} + \tau > \maxd + \tau/2$ for $\tau > 2\maxd$, one gets that 
\begin{align*}
\left(\sum_{j = 1}^n \frac{P_{ij}(1 - P_{ij})}{\sd_{j,0} + \tau}\right)  \leq \frac{\sd_{i,0}}{\maxd + \tau/2}.
\end{align*}
for $\tau > 2\maxd$. Consequently, using $\sd_{i,0}/(\sd_{i,0} + \tau) \leq \maxd/(\maxd + \tau)$, one gets that $\sigma^2 \leq \maxd/(\maxd + \tau/2)^2$ for $\tau > 2\maxd$.

Applying Corollary 4.2 in \cite{mackey2012matrix} one gets
$$P\left(\|\tilde L_0 - \ml_0\| \geq t\right) \leq n e^{-t^2/2\sigma^2}.$$ Consequently, with probability at least $1 - 1/n^{c_1 - 1}$ one has,
$$\|\tilde L - \ml\| \leq \sqrt{\frac{2c_1\log n}{\mind}}.$$
Thus, with probability at least $1 - 1/n^{c_1 - 1}$, one has
\begin{equation}
\|\tilde L - \ml\| \leq \sqrt{2c_1 \log n}\,\, \tepst. \label{eq:secondsecondconclap}
\end{equation}

As a result, combining \eqref{eq:secondfirstconclap} and \eqref{eq:secondsecondconclap}, one gets that with probability at least $1 - 2/n^{c_1 -1}$, one has
$$\|L_\tau - \ml_\tau \| \leq \sqrt{\log n}\,\, \tepst \left[\sqrt{2c_1 } + c_2  \left(2 + (c_2/\sqrt{c}) \right)\right] $$
Substituting the values of $c_2,\, c$, and noting that $c_1  > 2$ one gets the expression in the theorem.
\end{proof}

\subsection{Proof of Lemma \ref{lem:rescue}} \label{subsec:pflemmarescue}

Notice that the population regularized Laplacian $\ml_\tau$ corresponds to the population Laplacian of an ordinary stochastic block model with block probability matrix
$$\mb_\tau = \mb + v v',$$ where $v = (\sqrt{\tau/n})\mathbf{1}$. Correspondingly, we  can use the following facts of the population eigenvectors and eigenvalues given for a SBM.

Let $Z$ be the community membership matrix, that is, the  $n\times K$ matrix with entry $(i,k)$ being 1 if node $i$ belongs to cluster $C_k$.  The following is proved in \cite{rohe2011spectral} :
\begin{enumerate}
 \item Let $\mr = \md_\tau^{-1} $. Then, the non-zero eigen values of 
$\ml_\tau$ are the same as that of 
\begin{equation}
\beg = \mb_\tau (Z'\mr Z), \label{eq:tmb}
\end{equation}
or equivalently, $\tmb = (Z'\mr Z)^{1/2} \mb_\tau (Z'\mr Z)^{1/2}$.
\item Define $\mu = R^{1/2} Z(Z'\mr Z)^{-1/2}$. Let,
$$\tmb = H\Lambda H^T,$$
where the right side of the above gives the singular value decomposition of the matrix on the right.   Then the eigenvectors of $\ml_\tau$ are given by $\mu H$.
\end{enumerate} 
Further, since in the stochastic block model the expected node degrees are the same for all nodes in a particular cluster, one can write $R^{1/2} Z = Z Q$, where $Q^{-2}$ is the $K \times K$ diagonal matrix of population degrees of nodes in a particular community. Consequently, one sees that
$$\mu H = Z(Z^T Z)^{-1/2} H.$$
Lemma \ref{lem:rescue}  follows from noting that
$$\mu H (\mu H)^T = Z (Z^T Z)^{-1} Z^T$$
and the fact that $(Z^T Z)^{-1} = diag(1/\nn{1}, \ldots, 1/\nn{k})$. 	


\section{Proof of Theorem \ref{thm:kblockgenthm} } \label{sec:proofkblockgenthm}


We first prove \eqref{eq:asympfrac}. Recall that $\ddeltn$ is the limit of $\epst/\mu_{K,\tau}$, as $\tau \rightarrow \infty$. Now $\tau \epst$ converges to $20\sqrt{\maxd \log n}$. Consequently, we now show that 
\begin{equation}
\lim_{\tau \rightarrow \infty} \frac{1}{\tau \mu_{K,\tau}} \asymp \frac{\tilde{m}_{1,n} m_{1,n} - m_{2,n}}{m_{1,n}}. \label{eq:www1}
\end{equation}
Recall that $\mu_{K,\tau}$ is the $K$-th smallest eigenvalue of $\beg$ \eqref{eq:tmb}. Now,
\begin{equation*}
\mu_{K,\tau} \asymp \frac{1}{trace(\beg^{-1})}. 
\end{equation*}
The above follows from noting that $\mu_{K,\tau}$ is also equal to the inverse of the largest eigenvalue of $\beg^{-1}$, and the fact that the latter is $\asymp trace(\beg^{-1})$, as $K$ is fixed. We now proceed to show that $trace(\beg^{-1})/\tau$ converges to a quantity that is of the same order of magnitude as the right side of \eqref{eq:www1}. This will prove \eqref{eq:asympfrac}.

Recall that the block probability matrix $\mb$ is given by \eqref{eq:bdiag}.  We first consider the case that $q_n =0$, that is, there is no interaction between the clusters. Notice,
 $$\beg^{-1} = F^{-1} (\mb + v v')^{-1},$$
 where
 $$F^{-1} = diag\left( \frac{\cldegt{1} + \tau}{\nn{1}}, \ldots, \frac{\cldegt{K} + \tau}{\nn{k}}\right).$$
Here, for convenience, we remove the subscript $n$ from quantities such as $\cldeg{i}$. 
Using Sherman-Morrison formula
$$(\mb + v v')^{-1} = \mb^{-1}  - \frac{(\mb^{-1}v)(\mb^{-1}v)'}{1 + v' \mb^{-1} v}$$
One sees that, $\mb^{-1} v = \sqrt{\tau/n}(1/p_1, \ldots,1/p_K)'$.
Correspondingly, $$v'\mb^{-1} v = \frac{\tau}{n}\sum_{i} 1/p_1 = \tau \modinv,$$  
using $q_n = 0$. 
Further, the diagonal entries of the matrix $(\mb^{-1}v)(\mb^{-1}v)'$ can be written as
$$\frac{\tau}{n}diag(1/p_1^2,\ldots, 1/p_K^2).$$
We need the trace of $\beg^{-1}$. Using the above, one sees that 
\begin{equation*}
trace(\beg^{-1}) 
=\sum_k \frac{\cldegt{k} + \tau}{\cldegt{k}} - \frac{\tau\modinv + \tau^2 \vodinv}{1 + \tau \modinv}.
\end{equation*}
Since $K$ is fixed, we have,
\begin{equation}
trace(\beg^{-1}) \asymp  \tau\tmone - \frac{\tau\modinv + \tau^2 \vodinv}{1 + \tau \modinv}. \label{eq:traceasym}
\end{equation}
Thus, as $\tau \rightarrow \infty$, one gets that,
$$\frac{trace(\beg^{-1})}{\tau} \mbox{ converges to } \tmone - \mtwo/\mone. $$
The right side of the above is positive, as $\tmone\mone \geq \mtwo$, for $K > 1$. 

Now consider the $K$ block model with off-diagonal elements of $\mb$ equal to $q$. Notice that $$\mb_\tau = \mb_0 + \tilde v (\tilde v)^T,$$where $\mb_0 = diag(p_1 -q ,\ldots, p_K -q)$ 
and $\tilde v = \sqrt{\tilde \tau/n}\mathbf{1}$, where $\tilde \tau = \tau + nq$.  
Thus applying the above result for the diagonal block model one gets that if $\tau$ tends to infinity, the quantity $trace(\beg^{-1})/\tau$ converges to $\tmone - \mtwo/\mone$, where here $\cldegt{k} = \nn{k} (p_k - q)$. This proves \eqref{eq:asympfrac}.

%
%

We now prove that RSC-$\tau_n$ provides consistent cluster estimates for $\{\tau_n, \, n\geq 1\}$ satisfying \eqref{eq:tausat1}. We need to show that $\epsilon_{\tau_n,n}/\mu_{K,\tau_n}$ goes to zero. 

First, notice that $\tau_n\epsilon_{\tau_n,n}  \lesssim \sqrt{\maxd \log n}$. Consequently, from the above, we need to show that $ trace(\beg^{-1})\sqrt{\maxd \log n}/\tau_n$ is $o(1)$ if $\ddelta_n = o(1)$. From \eqref{eq:traceasym} one has
$$\frac{trace(\beg^{-1})\sqrt{\maxd \log n}}{\tau_n} \asymp  \sqrt{\maxd \log n}\left[\frac{\tmone - \modinv}{1 + \tau_n \modinv}  + (\tau_n\modinv)\frac{\tmone  - m_{2,n}/\modinv}{1 + \tau_n \modinv}\right] $$
The second term is bounded by $\ddelta_n$, which, by assumption, goes to zero.  The first term is bounded by $\sqrt{\maxd \log n}\,\tmone/(\modinv\tau_n)$. Noting that $\tmone/\modinv \lesssim \sum_{k}1/w_k$, one gets that the second terms also goes to 0, as $\tau_n$ satisfies \eqref{eq:tausat1}.

\subsection{Proof of Corollary \ref{cor:twoblock}}

 For the $K$-block SBM, let $r_K = \cldeg{K}/\cldeg{K-1}$. Notice that $r_K \asymp (p_{K-1} - q)/(p_K - q)$ using $w_k \asymp 1$.
 Use the fact that $m_{1,n} = (1/\cldeg{K})(w_K + O(r_K))$, $\tilde{m}_{1,n} = (1/\cldeg{K})(1 + O(r_K))$ and $m_{2,n} = (1/\cldeg{K}^2)(w_K + O(r_K))$, to get that
$$ \frac{\left(\tmone \mone - \mtwo\right)}{\mone}  =O(1/\cldeg{K-1}).$$
Consequently, $\ddeln = O(\sqrt{\maxd \log n}/\cldeg{K-1}) $. The proof of claim \eqref{eq:fracmistwoblck} is completed by noting that $\cldeg{K-1} \asymp n\,(p_{K-1} - q)$ and $\maxd \asymp n\, p_{1,n}$.

For the 2-block SBM we show that
\begin{equation}
\ddeln \asymp \frac{\sqrt{\maxd \log n}}{w_1 w_2\left[(p_{1,n} + p_{2,n})/2 - q_n\right]}. \label{eq:fntwoblck}
\end{equation}

 Expression \eqref{eq:fntwoblck} follows from using \eqref{eq:asympfrac} and noting that
$$\frac{\left(\tmone \mone - \mtwo\right)}{\mone} =  \frac{1}{w_2 \cldeg{1} + w_1 \cldeg{2}}$$
for the two-block model. It is seen that
 $$w_2 \cldeg{1} + w_1 \cldeg{2} = 2n\,w_1 w_2\left[ \left(p_{1,n} + p_{2,n}\right)/2 - q_n\right].$$
 Notice that $w_1 w_2 \asymp \min\{w_1,\, w_2\}$. Consequently,  \eqref{eq:fracmistwoblckp} follows from noting that when $p_{2,n} = q_n$ then $\maxd = n(w_1 p_{1,n} + w_2 q_{n})$.

\section{Proof of Results in Section \ref{sec:selhigh} } 
 \label{sec:thmhighdeg}

In this section we provide the proof Theorem \ref{thm:highdegcor1}, along with Lemmas \ref{lem:eganalkp1} and  \ref{lem:start} required in proving the theorem.  

\subsection{Proof of Theorem \ref{thm:highdegcor1}} \label{subsec:proofhigh}
%
%

 Denote $C^w$ as the set of nodes belonging to the weak clusters. We club all the nodes belonging to  the weak clusters into the cluster $C_K$ and call this combined cluster as $\tilde{C}_K$, that is $\tilde{C}_K = C_K \cup C^w$. For consistency of notation, let $\tilde{C}_k = C_k$, for $1\leq k \leq K-1$, and let $\tnn{k} = |\tilde{C}_k|$, for $k = 1,\,\ldots,\,K$.  

Denote 
  \begin{equation*}
\tilde f = \min_\pi \max_k \frac{|\tilde{C}_k \cap \hat T_{\pi(k)}^c| + |\tilde{C}_k^c \cap \hat T_{\pi(k)}|}{\tnn{k}}.
\end{equation*}
It is not hard to see that,
$$\hat f \leq \left(1 + \frac{n^w}{\nk}\right) \tilde f + \frac{n^w}{\nk}.$$
Consequently, a demonstration the $\tilde f$ goes to zero, along with the fact that $n^w = O(1)$, will show that $\hat f$ goes to zero.

 We now show that $\tilde f$ goes to zero with high probability. For a given assignment of nodes in one of the $K + \kw$ clusters we denote $L_\tau,\, \ml_\tau$ to be the sample, population regularized Laplacians respectively. 
  Further, let $\tml_\tau$ be the population regularized Laplacian of a $K+1$-block SBM constructed from clusters $C_1,\, \ldots, C_K$ and $C^w$, and block probability matrix
\begin{equation*}
 \tmbb =  \left( \begin{array}{cc}
B_s & b_{sw} \one  \\
b_{sw}\one' \one & 1 \\
\end{array} \right),
\end{equation*}
where the $K\times K$ matrix $B_s$, as in Section \ref{sec:selhigh}.

Since $\tmbb$ has rank $K+1$, the same holds also for $\tml_\tau$. We denote by $\tilde{\mu}_{k,\tau}$, for $k = 1,\ldots, n$, to be the magnitude of the eigenvalues of $\tml_\tau$ arranged in decreasing order.  Notice that $\tilde{\mu}_{k,\tau} = 0$ for $k > K+1$.  Further,  let $\vsym_\tau$ be the $n\times K$   eigenvector matrix of $\tml_\tau$.

Lemma \ref{lem:eganalkp1} shows that $\tilde{\mu}_{2,\tau} = \ldots = \tilde{\mu}_{K,\tau}$, as well as provides explicit expression for these eigenvalues. Further, the lemma also characterizes the norm of the difference of the rows of $\vsym_\tau$. In the lemma below we denote by $\dsn = \nk \pkn + (n - K\nk) q_n + n^w b_{sw}$ and $\dwn = n^w  + (n -  n^w) b_{sw}$. The quantities $\dsn$ and $\dwn$ provide the expected degrees of the nodes for an SBM drawn according to $\tmbb$.
 
 \begin{lem} \label{lem:eganalkp1} The following holds:
 \begin{enumerate}
\item   The eigenvalue $\tilde{\mu}_{1,\tau} = 1$. Further, let $\cldegt{n} = \nk (\pkn - q_n)$.  Then
\begin{align}
\tilde{\mu}_{k,\tau} &= \frac{\cldegt{n}}{\dsn + \tau} \quad \mbox{for $k = 2,\ldots, K $} \label{eq:repeigen}\\
\tilde{\mu}_{K+1,\tau} &= \frac{n^w (1 + \tau/n)}{\dwn + \tau}\,\, -\,\, \frac{n^w (b_{sw} + \tau/n)}{\dsn + \tau} \label{eq:nonrepeg}.
\end{align}
\item  The matrix $\vsym_\tau$ has $K+1$ distinct rows corresponding to the $K+1$ clusters $C_1,\ldots, C_K$ and  $C^w$.  Denote these as $\cent_{1,\tau},\ldots,\,\cent_{K,\tau}$ and $\cent_\tau^w$.\\ Then $1\leq k' \neq k \leq K$
            $$\|\cent_{k,\tau} - \cent_{k',\tau}\| = \sqrt{\frac{2}{\nk}} $$
            for $1\leq k \leq K$,
                        $$\|\cent_{k,\tau} - \cent_\tau^w\| = \sqrt{\frac{1}{\nk}}$$
 \end{enumerate}
 \end{lem}

 The above lemma is proved in Appendix \ref{sec:kblock}.
  Let 
 $\tvsym_\tau$ be an $n\times K$ matrix, with
 $$\tvsym_{i,\tau} = \cent_{k,\tau}\quad\mbox{for $i \in \tilde{C}_k$}.$$
 Now $\tvsym_\tau$ has $K$ distinct rows corresponding to the $K$ clusters $\tilde{C}_1,\ldots,\tilde{C}_K$. We denote these distinct rows as the population cluster centers. From Lemma \ref{lem:kumar}, if 
 $$\|\cent_{k,\tau} - \cent_{k',\tau}\| \gtrsim (1/\ddelta) \|V_{\tau} - \tvsym_{\tau}\|/\sqrt{\nk},$$
 then $\tilde f =O(\ddelta^2)$.
 Since $\|\cent_{k,\tau} - \cent_{k',\tau}\| \asymp 1/\sqrt{\nk}$ from Lemma \ref{lem:eganalkp1}, one gets that one needs to show that 
 $\|V_{\tau} - \tvsym_{\tau}\| \lesssim \ddelta$,
 with high probability, for some $\ddelta$ that goes to zero for large $n$. 
 
 Now,
 \begin{align}
 \|V_\tau - \tvsym_{\tau}\| &\leq \|V_\tau - \vsym_\tau\| + \|\vsym_\tau - \tvsym_\tau\|\nonumber\\
 &= \|V_\tau - \vsym_\tau\| +\sqrt{\frac{n^w}{\nk}} \nonumber
 \end{align}
  
 As $n^w = O(1)$, one needs to show that $\|V_\tau - \vsym_\tau\|$ goes to zero with high probability. From Davis-Kahan theorem we get that
\begin{align}
\|V_\tau - \vsym_\tau\| &\lesssim \frac{\|L_\tau - \tml_\tau \|}{\tilde{\mu}_{K,\tau} - \tilde{\mu}_{K+1,\tau}}\nonumber \\
&\lesssim \frac{\|L_\tau - \ml_\tau\| + \|\ml_\tau - \tml_\tau\| }{\tilde{\mu}_{K,\tau} - \tilde{\mu}_{K+1,\tau}}\label{eq:partdone}
\end{align}
The following lemma shows that for large $\tau$, the Laplacian matrix $\ml_{\tau}$ is close to the Laplacian matrix $\tml_{\tau}$ in spectral norm.
 \begin{lem}  \label{lem:start} 
  $$\|\ml_\tau - \tml_\tau\| \lesssim \frac{1}{1 + \tau/\dwn} $$
\end{lem}
The lemma is proved in Appendix \ref{sec:startproof}.  Consequently, from Lemma \ref{lem:start} and Theorem \ref{thm:someassump}
one gets from \eqref{eq:partdone} that
\begin{align}
\|V_{\tau} - \vsym_{\tau}\| 
&\lesssim \frac{1}{(\tilde{\mu}_{K,\tau} - \tilde{\mu}_{K+1,\tau})} \left( \epst + \frac{1}{1 + \tau/\dwn} \right)\label{eq:eigenvecbdd}
\end{align}

 Further, from Lemma \ref{lem:eganalkp1} one gets that
$$\tilde{\mu}_{K,\tau} - \tilde{\mu}_{K+1,\tau} = \frac{\nk( \pkn - q_n)}{\dsn +\tau}   - \left[
 \frac{n^w (b_s + \tau/n)}{\dwn + \tau} - \frac{n^w (b_{sw}+ \tau/n)}{\dsn + \tau} \right]
$$
It is seen that $(\tilde{\mu}_{K,\tau} - \tilde{\mu}_{K+1,\tau}) \tau$ converges to 
$$\nk(\pkn - q_n)  - \left[n^w (b_s - b_{sw}) + (n^w/n) (\dsn - \dwn) \right],$$
which is $\gtrsim \nk(\pkn - q_n)$ using $n^w = O(1)$.  

Consequently,  the right side of \eqref{eq:eigenvecbdd} converges  to 
$$\frac{\sqrt{\dsn \log n} + \dwn}{\nk( \pkn - q_n)} $$
for large $\tau$.  Now, $\dsn \asymp n\, \pkn$ and $\dwn \asymp n\, b_{sw}$ (using $n^w = O(1)$). Consequently, the numerator in the above is $\lesssim \sqrt{n\pkn}$ using Assumption \eqref{eq:smallvsweak2}. 
Consequently, under Assumption \ref{eq:absence}, one gets that $\|V_\tau - \vsym_\tau\|$ goes to zero with high probability.

\subsection{Proof of Lemma \ref{lem:start}} \label{sec:startproof}
We bound the spectral norm of $\ml_{\tau} - \tml_{\tau}$. Here $\tml_{\tau}$ is as in Appendix \ref{subsec:proofhigh}. Take $\ml_\tau =  \md^{-1/2} \left(P + (\tau/n) J\right) \md^{-1/2}$ and $\tml_\tau  =  \tmd^{-1/2} \left(\tp + (\tau/n) J\right) \tmd^{-1/2}$. Notice that we ignore the subscript $\tau$ in both $\md$ and $\tmd$. Here, $\tp = Z\tmbb Z'$, with $\tmbb$ as in Subsection \ref{subsec:proofhigh}.

As in the proof of Theorem \ref{thm:someassump}, given in Appendix \ref{subsec:conclap}, write $$\ml_\tau' = \tmd^{-1/2} \left( P + (\tau/n) J\right) \tmd^{-1/2}.$$ Then,
\begin{equation}
\|\ml_\tau - \tml_\tau\| \leq \|\ml_\tau - \ml_\tau'\| + \|\ml_\tau' - \tml_\tau\|. \label{eq:pertpop}
\end{equation}
Consequently, we prove that $\ml_\tau$ is close to $\tml_\tau$ by showing that both terms in the right side of \eqref{eq:pertpop} are small. We first bound $\| \ml_\tau - \ml_\tau'\|$. As in \eqref{eq:implater}, write
$$\| \ml_\tau - \ml_\tau'\| \leq \| I - F\| \left(2 + \| I - F\|\right),$$
where as before $F - I = \left(I + (\md - \tmd) \tmd^{-1} \right)^{1/2} - I$. Here
 $\md = diag(\sd_{1,\tau},\, \ldots,\, \sd_{n,\tau})$, and $\tmd = diag(\tilde{\sd}_{1,\tau},\, \ldots,\, \tilde{\sd}_{n,\tau})$. Now,
\begin{align*}
\|(\md - \tmd ) \tmd^{-1}\| &\leq \frac{|\sd_{i,\tau} - \tilde{\sd}_{i,\tau}|}{\tsd_{i,\tau}} \\
&\lesssim\frac{\dwn}{\left(\dwn + \tau\right)}.
\end{align*}
Observe that we can assume that $\|(\md - \tmd) \tmd^{-1}\| \leq 3/4$ for large $\tau$, so that $$\left(1 +\| (\md - \tmd) \tmd^{-1} \|\right)^{1/2} - 1 \leq \| (\md - \tmd) \tmd^{-1} \|,$$
and thus $\| \ml_\tau- \ml_\tau'\| \lesssim \dwn/(\dwn + \tau)$.

Next, we bound $\|\ml_\tau' - \tml_\tau\|$. Notice that
$\|\ml_\tau' - \tml_\tau\| \leq  \left\|\tmd^{-1}\right\|\left\|\left( P - \tp \right)\right\|.$  The quantity $\|\tmd^{-1} \| \lesssim 1/(\dwn + \tau)$. Further, note that $\|P - \tp\| \lesssim \dwn$, since $P - \tp$ is a matrix with all entries negative and hence its spectral norm is at most the maximum of its row sums.

\subsection{Proof of Lemma \ref{lem:eganalkp1}}\label{sec:kblock}

We investigate the eigenvalues of the  $K+1$ community stochastic block model with block probability matrix
\[ \tmbb =  \left( \begin{array}{cc}
B_s & b_{sw} \one  \\
b_{sw}\one' \one' & b_w \\
\end{array} \right)\] 
In our case $b_w = 1$. Denote the corresponding population Laplacian by $\tml$.  
 Recall that from Subsection \ref{subsec:pflemmarescue} the non-zero eigenvalues of $\ml$ are the same as that of 
$$\tmb = (Z'\mr Z)^{1/2}\mb (Z'\mr Z)^{1/2}$$
Now,
$$ Z'\mr Z = diag\left(\frac{\nk}{\dsn}, \ldots , \frac{\nk}{\dsn}, \, \frac{n^w}{\dwn}   \right)$$
Consequently,
\renewcommand{\arraystretch}{2.5}
\[ \tmb =  \left( \begin{array}{cc}
\frac{\nk}{\dsn}B_s & \left(\frac{\nk n^w}{\dsn\dwn}\right)^{1/2} b_{sw} \one  \\
\left(\frac{\nk n^w}{\dsn\dwn}\right)^{1/2} b_{sw} \one' & \frac{n^w}{\dwn} b_w \\
\end{array} \right),\] 
One sees that 
$$v_1 = (\sqrt{\nk\dsn}, \,\ldots,\,\sqrt{\nk\dsn},\, \sqrt{n^w\dwn})'$$
is an eigenvector of $\tmb$ with eigenvalue 1.
Next, consider a vector $v_2 = (v_{21}',0)'$. Here $v_{21}$ is a $K\times 1$ dimensional vector that is orthogonal to the constant vector.  We claim that $v_2$ so defined is also an eigenvector of $\tmb$. To see this notice that 
\renewcommand{\arraystretch}{1.5}
 \[ \tmb\, v_2 =  \frac{\nk}{\dsn}\left( \begin{array}{c}
B_s v_{21}  \\
0  \\
\end{array} \right),\] 
Here we use the fact that $\one' v_{21} = 0$ as $v_{21}$ is orthogonal to $\one$.  Next, notice that 
$$B_s =  \left( (\pkn - q_{n})I + q_{n} \one \one'  \right)$$
Consequently, 
$$B_s v_{21} =   (\pkn - q_{n}) v_{21}$$
The above implies that $v_{2}$ is an eigenvector of $\tmb$ with eigenvalue  $\egs_1$ given by $n^s(p_n^s - q_n)/d_n^s$.

Notice that from the above construction one can get  $K-1$ orthogonal eigenvectors $v_k$, for $k = 2,\ldots, K$, such that the  $v_k$'s are also orthgonal to $v_1$. Essentially,  for $k \geq 2$, each $v_k = (v_{k1}',0)'$, where $v_{k1}'\one = 0$.  There are $K-1$ orthogonal choices of the $v_{k1}$'s.

Given that 1 and $\egs_1$ are eigenvalues of $\tmb$, with the latter having multiplicity $K-1$, the remaining eigenvalue is given by
\begin{align*}
\egs_2 &= trace(\tmb) - 1 - (K-1)\egs_1\\
    &=\frac{\nk \pkn}{\dsn} +  (K-1)\frac{\nk}{\dsn}  q_{n} + \frac{n^w b_w}{\dwn} - 1 \\
        &= \frac{n^w b_w}{\dwn} - \frac{n^w b_{sw}}{\dsn}.
\end{align*}
The claim regarding the eigenvector corresponding to $\egs_2$ follows from seeing that this should be the case since it is orthogonal to eigenvectors $v_1,\,\ldots, v_K$ defined above.

\section{Extending $DKest$ to allow for degree heterogeneity} \label{sec:extdkest}

Here, we describe  how we extend the $DKest$ by substituting the estimate $\hat \ml_\tau$ in \eqref{eq:dkeststatest} with one assuming that the data is drawn from a degree corrected stochastic block model (D-SBM). 
As mentioned before, the D-SBM is a more appropriate model for modeling network datasets with extremely heterogeneous node degrees. The edge probability matrix takes the form
$$P = \Theta Z \mb Z' \Theta,$$
 where $\Theta = diag(\theta_1,\, \ldots,\, \theta_n)$ models the heterogeneity in the degrees. 

As before, assume that $\hat C_{1,\tau},\ldots, \hat C_{K,\tau}$ be the cluster estimates obtained from running RSC-$\tau$ Algorithm. Let $\hat Z$ be the corresponding $n \times K$ cluster membership matrix. 
Denote $$\hat b_{k_1,k_2} = \sum_{i \in \hat C_{k_1, \tau}, \,\,\,  j \in \hat C_{k_2, \tau}} A_{ij}$$
and let $\hat B = (( \hat b_{k_1,k_2} ))$ be the $K\times K$ with entries $\hat b_{k_1,k_2}$.  

 As in \citet{karrer2011stochastic}, we  produce an estimate of the edge probability matrix $P$ given by
 $$\hat P = \hat \Theta \hat Z \hat B \hat Z' \hat \Theta,$$
  where $\hat \Theta = diag(\hat \theta_1,\, \ldots,\, \hat \theta_n)$, 
  with $$\hat \theta_i	 = \frac{\hat d_i}{\sum_{k'=1}^K \hat b_{k,k'}}$$
for $i \in \hat C_{k,\tau}$. Recall that $\hat d_i$ is the degree of node $i$.  It is seen that with the above definition of $\Theta$ the sum of the $i$-th row $\hat P$ is simply $\hat d_i$.  

The estimate $\hat \ml_\tau$ is taken as the population regularized Laplacian corresponding to the estimated edge probability matrix $\hat P$. In other words,
$$\hat \ml_\tau = \left(D + \tau I\right)^{-1/2} \left(\hat P + \frac{\tau}{n} \mathbf{1}\mathbf{1}'\right)\left(D + \tau I\right)^{-1/2},$$
where recall that $D$ is the diagonal matrix of degrees.

\bibliographystyle{plainnat}
\bibliography{speclust}

\begin{thebibliography}{27}
\providecommand{\natexlab}[1]{#1}
\providecommand{\url}[1]{\texttt{#1}}
\expandafter\ifx\csname urlstyle\endcsname\relax
  \providecommand{\doi}[1]{doi: #1}\else
  \providecommand{\doi}{doi: \begingroup \urlstyle{rm}\Url}\fi

\bibitem[Adamic and Glance(2005)]{adamic2005political}
Lada~A Adamic and Natalie Glance.
\newblock The political blogosphere and the 2004 us election: divided they
  blog.
\newblock In \emph{Proceedings of the 3rd international workshop on Link
  discovery}, pages 36--43. ACM, 2005.

\bibitem[Amini et~al.(2013)Amini, Chen, Bickel, and Levina]{chen2012fitting}
A.A. Amini, A.~Chen, P.J. Bickel, and E.~Levina.
\newblock Pseudo-likelihood methods for community detection in large sparse
  networks.
\newblock \emph{Ann. Statist}, 41\penalty0 (4):\penalty0 2097--2122, 2013.

\bibitem[Awasthi and Sheffet(2012)]{awasthi2012improved}
Pranjal Awasthi and Or~Sheffet.
\newblock Improved spectral-norm bounds for clustering.
\newblock In \emph{Approximation, Randomization, and Combinatorial
  Optimization. Algorithms and Techniques}, pages 37--49. Springer, 2012.

\bibitem[Belkin and Niyogi(2003)]{belkin2003laplacian}
Mikhail Belkin and Partha Niyogi.
\newblock Laplacian eigenmaps for dimensionality reduction and data
  representation.
\newblock \emph{Neural computation}, 15\penalty0 (6):\penalty0 1373--1396,
  2003.

\bibitem[Bhatia(1997)]{bhatia1997matrix}
Rajendra Bhatia.
\newblock \emph{Matrix analysis}, volume 169.
\newblock Springer, 1997.

\bibitem[Bickel and Chen(2009)]{bickel2009nonparametric}
Peter~J Bickel and Aiyou Chen.
\newblock A nonparametric view of network models and newman--girvan and other
  modularities.
\newblock \emph{Proceedings of the National Academy of Sciences}, 106\penalty0
  (50):\penalty0 21068--21073, 2009.

\bibitem[Chaudhuri et~al.()Chaudhuri, Chung, and
  Tsiatas]{chaudhuri2012spectral}
K.~Chaudhuri, F.~Chung, and A.~Tsiatas.
\newblock Spectral clustering of graphs with general degrees in the extended
  planted partition model.
\newblock \emph{Journal of Machine Learning Research}, 2012:\penalty0 1--23.

\bibitem[Chen et~al.(2012)Chen, Amini, Bickel, and Levina]{chenJSM}
A.~Chen, A.~Amini, P.~Bickel, and L.~Levina.
\newblock Fitting community models to large sparse networks.
\newblock In \emph{Joint Statistical Meetings, San Diego}, 2012.

\bibitem[Dhillon(2001)]{dhillon2001co}
Inderjit~S Dhillon.
\newblock Co-clustering documents and words using bipartite spectral graph
  partitioning.
\newblock In \emph{Proc. seventh ACM SIGKDD inter. conf. on Know. disc. and
  data mining}, pages 269--274. ACM, 2001.

\bibitem[Fishkind et~al.(2013)Fishkind, Sussman, Tang, Vogelstein, and
  Priebe]{fishkind2013consistent}
Donniell~E Fishkind, Daniel~L Sussman, Minh Tang, Joshua~T Vogelstein, and
  Carey~E Priebe.
\newblock Consistent adjacency-spectral partitioning for the stochastic block
  model when the model parameters are unknown.
\newblock \emph{SIAM Journal on Matrix Analysis and Applications}, 34\penalty0
  (1):\penalty0 23--39, 2013.

\bibitem[Hagen and Kahng(1992)]{hagen1992new}
Lars Hagen and Andrew~B Kahng.
\newblock New spectral methods for ratio cut partitioning and clustering.
\newblock \emph{IEEE Trans. Computer-Aided Design}, 11\penalty0 (9):\penalty0
  1074--1085, 1992.

\bibitem[Holland et~al.(1983)Holland, Laskey, and
  Leinhardt]{holland1983stochastic}
Paul~W Holland, Kathryn~Blackmond Laskey, and Samuel Leinhardt.
\newblock Stochastic blockmodels: First steps.
\newblock \emph{Social networks}, 5\penalty0 (2):\penalty0 109--137, 1983.

\bibitem[Joseph and Barron(2013)]{barron2010ajoseph}
A.~Joseph and A.R. Barron.
\newblock Fast sparse superposition codes have near exponential error
  probability for {$R < C$}.
\newblock \emph{IEEE. Trans. Inform. Theory, to appear}, 2013.

\bibitem[Karrer and Newman(2011)]{karrer2011stochastic}
Brian Karrer and Mark~EJ Newman.
\newblock Stochastic blockmodels and community structure in networks.
\newblock \emph{Physical Review E}, 83\penalty0 (1):\penalty0 016107, 2011.

\bibitem[Kumar and Kannan(2010)]{kumar2010clustering}
Amit Kumar and Ravindran Kannan.
\newblock Clustering with spectral norm and the k-means algorithm.
\newblock In \emph{Foundations of Computer Science (FOCS), 2010 51st Annual
  IEEE Symposium on}, pages 299--308. IEEE, 2010.

\bibitem[Kwok et~al.(2013)Kwok, Lau, Lee, Gharan, and
  Trevisan]{kwok2013improved}
Tsz~Chiu Kwok, Lap~Chi Lau, Yin~Tat Lee, Shayan~Oveis Gharan, and Luca
  Trevisan.
\newblock Improved cheeger's inequality: Analysis of spectral partitioning
  algorithms through higher order spectral gap.
\newblock \emph{arXiv preprint arXiv:1301.5584}, 2013.

\bibitem[Mackey et~al.(2012)Mackey, Jordan, Chen, Farrell, and
  Tropp]{mackey2012matrix}
L.~Mackey, M.I. Jordan, R.Y. Chen, B.~Farrell, and J.A. Tropp.
\newblock Matrix concentration inequalities via the method of exchangeable
  pairs.
\newblock \emph{arXiv preprint arXiv:1201.6002}, 2012.

\bibitem[McSherry(2001)]{mcsherry2001spectral}
Frank McSherry.
\newblock Spectral partitioning of random graphs.
\newblock In \emph{Foundations of Computer Science, 2001. Proceedings. 42nd
  IEEE Symposium on}, pages 529--537. IEEE, 2001.

\bibitem[Newman and Girvan(2004)]{newman2004finding}
Mark~EJ Newman and Michelle Girvan.
\newblock Finding and evaluating community structure in networks.
\newblock \emph{Physical review E}, 69\penalty0 (2):\penalty0 026113, 2004.

\bibitem[Ng et~al.(2002)Ng, Jordan, Weiss, et~al.]{ng2002spectral}
Andrew~Y Ng, Michael~I Jordan, Yair Weiss, et~al.
\newblock On spectral clustering: Analysis and an algorithm.
\newblock \emph{Advances in neural information processing systems}, 2:\penalty0
  849--856, 2002.

\bibitem[Oliveira(2009)]{oliveira2009concentration}
R.I. Oliveira.
\newblock Concentration of the adjacency matrix and of the laplacian in random
  graphs with independent edges.
\newblock \emph{arXiv preprint arXiv:0911.0600}, 2009.

\bibitem[Qin and Rohe(2013)]{qin2013regularized}
Tai Qin and Karl Rohe.
\newblock Regularized spectral clustering under the degree-corrected stochastic
  blockmodel.
\newblock \emph{arXiv preprint arXiv:1309.4111}, 2013.

\bibitem[Rohe et~al.(2011)Rohe, Chatterjee, and Yu]{rohe2011spectral}
K.~Rohe, S.~Chatterjee, and B.~Yu.
\newblock Spectral clustering and the high-dimensional stochastic blockmodel.
\newblock \emph{The Annals of Statistics}, 39\penalty0 (4):\penalty0
  1878--1915, 2011.

\bibitem[Shi and Malik(2000)]{shi2000normalized}
Jianbo Shi and Jitendra Malik.
\newblock Normalized cuts and image segmentation.
\newblock \emph{IEEE Trans. Pat. Analysis and Mach. Intel.}, 22\penalty0
  (8):\penalty0 888--905, 2000.

\bibitem[Sussman et~al.(2012)Sussman, Tang, Fishkind, and
  Priebe]{sussman2012consistent}
Daniel~L Sussman, Minh Tang, Donniell~E Fishkind, and Carey~E Priebe.
\newblock A consistent adjacency spectral embedding for stochastic blockmodel
  graphs.
\newblock \emph{Journal of the American Statistical Association}, 107\penalty0
  (499):\penalty0 1119--1128, 2012.

\bibitem[Von~Luxburg(2007)]{von2007tutorial}
Ulrike Von~Luxburg.
\newblock A tutorial on spectral clustering.
\newblock \emph{Statistics and computing}, 17\penalty0 (4):\penalty0 395--416,
  2007.

\bibitem[Yao(2003)]{yao2003information}
YY~Yao.
\newblock Information-theoretic measures for knowledge discovery and data
  mining.
\newblock In \emph{Entropy Measures, Maximum Entropy Principle and Emerging
  Applications}, pages 115--136. Springer, 2003.

\end{thebibliography}

\end{document}